\documentclass{article}
\usepackage[utf8]{inputenc}

\usepackage{mdframed}
\usepackage{algorithm,algorithm,multicol}
\usepackage{pgfplots}

\definecolor{Gred}{RGB}{219, 50, 54}
\definecolor{Ggreen}{RGB}{60, 186, 84}
\definecolor{Gblue}{RGB}{72, 133, 237}
\definecolor{Gyellow}{RGB}{247, 178, 16}
\definecolor{ToCgreen}{RGB}{0, 128, 0}
\definecolor{myGold}{RGB}{231,141,20}
\definecolor{myBlue}{rgb}{0.19,0.41,.65}
\definecolor{myPurple}{RGB}{175,0,124}
\definecolor{niceRed}{RGB}{153,0,0}
\definecolor{niceRed}{RGB}{190,38,38}
\definecolor{blueGrotto}{HTML}{059DC0}
\definecolor{royalBlue}{HTML}{057DCD}
\definecolor{navyBlueP}{HTML}{0B579C}
\definecolor{limeGreen}{HTML}{81B622}
\definecolor{nicePink}{RGB}{247,83,148}

\usepackage{cmap}
\usepackage[T1]{fontenc}
\usepackage{bm}
\usepackage{amsmath}
\usepackage{amsfonts}
\usepackage{amssymb}
\usepackage{amsbsy}
\usepackage{amsthm}
\usepackage{bbm}

\usepackage{graphicx, ucs}

\usepackage{subcaption}
\usepackage{rotating}
\usepackage{float}
\usepackage{tikz}

\usepackage{algorithm, caption}
\usepackage[noend]{algpseudocode}
\usepackage{listings}

\usepackage{enumitem}

\usepackage{hyperref}
\hypersetup{
	colorlinks = true,
	urlcolor = {myPurple},
	linkcolor = {royalBlue},
	citecolor = {nicePink}
}
\usepackage[nameinlink]{cleveref}

\usepackage{multirow}
\usepackage{array}

\usepackage{chngcntr}

\counterwithin*{equation}{section}
\usepackage{chngcntr}
\usepackage{soul}
\usepackage{nicefrac}

\usepackage{fancyhdr}
\fancyheadoffset{0pt}
\def\compactify{\itemsep=0pt \topsep=0pt \partopsep=0pt \parsep=0pt}
\let\latexusecounter=\usecounter

\definecolor{myC}{rgb}{0, 255, 255}
\definecolor{myY}{rgb}{204, 204, 0}
\definecolor{myM}{rgb}{255, 0, 255}
\definecolor{secinhead}{RGB}{249,196,95}
\definecolor{lgray}{gray}{0.8}

\newtheorem{theorem}{Theorem} 
\newtheorem*{theorem*}{Theorem} 
\newtheorem*{proposition*}{Proposition} 
\newtheorem{lemma}{Lemma}
\newtheorem*{lemma*}{Lemma}
\newtheorem{claim}{Claim}
\newtheorem{proposition}{Proposition}

\newtheorem{corollary}{Corollary}

\newtheorem{question}{Question}

\newtheorem{definition}{Definition}

\newtheorem{remark}{Remark}

\renewcommand{\Pr}{\mathop{\bf Pr\/}}
\newcommand{\E}{\mathop{\bf E\/}}

\newcommand{\poly}{\textnormal{poly}}

\newcommand{\sgn}{\textnormal{sgn}}

\newcommand{\reals}{\mathbb R}

\newcommand{\nats}{\mathbb N}

\newcommand{\eps}{\varepsilon}

\newcommand{\TV}{\mathrm{TV}}

\newcommand{\calB}{\mathcal{B}}

\newcommand{\calD}{\mathcal{D}}
\newcommand{\calE}{\mathcal{E}}
\newcommand{\calF}{\mathcal{F}}

\newcommand{\calH}{\mathcal{H}}
\newcommand{\calI}{\mathcal{I}}

\newcommand{\calO}{\mathcal{O}}
\newcommand{\calP}{\mathcal{P}}

\newcommand{\calR}{\mathcal{R}}
\newcommand{\calS}{\mathcal{S}}

\newcommand{\calX}{\mathcal{X}}
\newcommand{\calY}{\mathcal{Y}}

\def\<{\langle}
\def\>{\rangle}

\def\wt{\widetilde}
\def\wh{\widehat}

\def\dtv{\mathrm{TV}}

\def\err{\mathrm{err}}

\def\dtv{d_{\mathrm{TV}}}

\makeatletter
\renewenvironment{abstract}{%
	\if@twocolumn
	\section*{\abstractname}%
	\else 
	\begin{center}%
		{\bfseries \large\abstractname\vspace{\z@}}
	\end{center}%
	\quotation
	\fi}
{\if@twocolumn\else\endquotation\fi}
\makeatother

\title{Statistical Indistinguishability of Learning Algorithms}
\author{
\begin{tabular}{cc}	
        \begin{tabular}{c}
        \text{Alkis Kalavasis} \\
		 National Technical University of Athens \\
		\url{kalavasisalkis@mail.ntua.gr}
        \end{tabular}
        & 	
		\begin{tabular}{c}
		\text{Amin Karbasi} \\
		 Yale University, Google Research \\
		\url{amin.karbasi@yale.edu}
		\end{tabular}
		\\
		\\
		\begin{tabular}{c}
		\text{Shay Moran} \\
		 Technion, Google Research \\
		\url{smoran@technion.ac.il}
		\end{tabular}
		 & 
		 \begin{tabular}{c}
		 \text{Grigoris Velegkas} \\
		 Yale University \\
		\url{grigoris.velegkas@yale.edu}\\
		\end{tabular}
		\end{tabular}
}
\date{\today}

\begin{document}

\maketitle
\date{}
\begin{abstract}

When two different parties use the same learning rule on their own data, how can we test whether the distributions of the two outcomes are similar? 
In this paper, we study the similarity of outcomes of learning rules 
through the lens of the
Total Variation (TV) distance of distributions.
We say that a learning rule is TV indistinguishable if the expected TV distance between the posterior distributions of its outputs, executed on two training data sets drawn independently from the same distribution, is small.  
We first investigate the learnability of hypothesis classes using TV indistinguishable learners. 
Our main results are information-theoretic equivalences between TV indistinguishability and existing algorithmic stability notions such as replicability and approximate differential privacy.
Then, we provide statistical amplification and boosting algorithms for TV indistinguishable learners. 
\end{abstract}

\section{Introduction}
Lack of replicability in experiments has been a major issue, usually referred to as the \emph{reproducibility crisis}, in many scientific areas such as biology and chemistry. 
Indeed, the results of a survey that appeared in Nature \cite{baker20161} are very worrisome: more than $70\%$ of the researchers that participated in it could not replicate other researchers' experimental findings while over half of them were not able to even replicate their own conclusions. In the past few years the number of scientific publications in the Machine Learning (ML) community has increased exponentially.
Significant concerns and questions regarding replicability have also recently been raised in the area of ML. This can be witnessed
by the establishment of various reproducibility challenges in major 
ML conferences such as the ICLR $2019$ Reproducibility Challenge \cite{pineau2019iclr}
and the NeurIPS $2019$ Reproducibility Program \cite{pineau2021improving}.

Reproducibility of outcomes in scientific research is a necessary
condition to ensure that the conclusions of the studies
reflect inherent properties of the underlying population and are not 
an artifact of the methods that scientists used
or the random sample of the population that the study was conducted on. 
In its simplest form, it requires  that if two different groups
of researchers carry out an experiment using the same methodologies
but \emph{different} samples of the \emph{same} population, it better
be the case that the two outcomes of their studies are \emph{statistically
indistinguishable}. In this paper, we investigate this notion in the context of ML (cf. \Cref{def:general indist of outcomes}), and characterize for which learning problems statistically
indistinguishable learning algorithms exist. Furthermore, we show how statistical indistinguishability, as a property of learning algorithms, is naturally related to various notions of algorithmic stability such as replicability of experiments, and differential privacy.

While we mainly focus on the fundamental ML task of binary classification to make the 
presentation easier to follow,
many of our results extend to other statistical tasks (cf. \Cref{apx:general definition 
ind}). More formally, the objects of interest are
\emph{randomized} learning rules
$A : (\calX \times \{0,1\})^n \to \{0,1\}^\calX$. These learning rules
    take as input a sequence $S$ of $n$ pairs from $\calX \times \{0,1\}$, i.e., points from a domain $\calX$ along with their labels, and map
them to a binary classifier in a randomized manner. We assume that this sequence $S$ is generated i.i.d. from a distribution $\calD$ on $\calX \times \{0,1\}$. We denote by $\{0,1\}^\calX$ the space of binary classifiers and by $A(S)$ 
the random variable that corresponds 
to the output of $A$ on input $S$\footnote{We identify with $A(S)$ the posterior distribution of $A$ on input~$S$ when there is no confusion.}. We also adopt a more algorithmic viewpoint for $A$ where we denote it as a \emph{deterministic} mapping $(\calX \times \{0,1\})^n \times \calR \to \{0,1\}^\calX$, 
which takes as input a training set $S$ of size $n$
made of instance-label pairs and a random string $r \sim \calR$ (we use $\calR$ for both the probability space and the distribution) corresponding to the algorithm's 
\emph{internal randomness}, and outputs a hypothesis $A(S,r)\in\{0,1\}^\calX.$ 
Thus, $A(S)$ corresponds to a random variable while $A(S,r)$ is a deterministic object. To make the distinction clear, 
we refer to $A(S)$ as (the image of) a \emph{learning rule}
and to $A(S,r)$ as (the image of) a \emph{learning algorithm}.

\paragraph{Indistinguishability.} We measure how much two distributions over hypotheses differ using some notion of \textbf{statistical dissimilarity} $d$,  
which can belong to a quite general class; we could let it be either an Integral Probability Metric (IPM) (e.g., TV or Wasserstein distance, see \Cref{def:ipm}) or an $f$-divergence (e.g., KL or Rényi divergence). For further details, see \cite{sriperumbudur2009integral}.
We are now ready to introduce the following general definition of \emph{indistinguishability of learning rules}.
\begin{definition}[Indistinguishability]
\label{def:general indist of outcomes}
    Let $d$ be 
    a statistical dissimilarity measure.
    A learning rule $A$ is
    $n$-sample 
    $\rho$-indistinguishable with respect to $d$
    if for any distribution $\calD$ over inputs 
    and two independent sets $S,S' \sim \calD^n$ it holds that
    \[
        \E_{S, S' \sim \calD^n} \left[d\left(A(S ),A(S')\right)\right] \leq \rho \,.
    \]
\end{definition}
In words, \Cref{def:general indist of outcomes} states that the expected dissimilarity of the outputs of the
learning rule when executed on two training sets that are drawn independently 
from $\calD$ is small. 
We view \Cref{def:general indist of outcomes}
as a general information-theoretic way to study indistinguishability as a property of learning rules. In particular, it captures the property that the distribution of outcomes
of a learning rule being
\emph{indistinguishable} under the resampling of its inputs. 
\Cref{def:general indist of outcomes} provides the flexibility
to define the dissimilarity measure according to the needs of
the application domain.
For instance, it captures as a special case the global stability property \cite{bun2020equivalence} (see \Cref{apx:general definition
ind}).
\paragraph{Replicability.} 
Since the issue of replicability is omnipresent in scientific disciplines it is important to design a formal framework through which we can argue about the replicability of experiments.
Recently, various works proposed algorithmic definitions of replicability
in the context of learning from samples \cite{impagliazzo2022reproducibility,bun2023stabilityIsStablest}, 
optimization \cite{ahn2022reproducibility}, bandits \cite{esfandiari2022reproducible} and clustering \cite{esfandiari2023replicable},
and designed algorithms that are provably replicable under these definitions. A notion that is closely related to \Cref{def:general indist of outcomes} was introduced by \cite{impagliazzo2022reproducibility}: reproducibility
or replicability\footnote{This property was originally defined as ``reproducibility'' in \cite{impagliazzo2022reproducibility}, but later it was pointed out that the correct term for this definition is ``replicability'' (see also \cite{bun2023stabilityIsStablest}). We use the term replicability throughout our work.} of learning algorithms is defined as follows:
\begin{definition}
[Replicability \cite{impagliazzo2022reproducibility}]
\label{def:replicability}
Let $\calR$ be a distribution over random strings.
A learning algorithm $A$ is 
$n$-sample 
$\rho$-replicable if 
for any distribution $\calD$ over inputs 
and two independent sets 
$S, S' \sim \calD^n$ it holds that
\[
\Pr_{S,S'\sim \calD^n, r \sim \calR}[A(S, r) \neq A(S', r) ] \leq \rho\,.
\]
\end{definition}

The existence of a shared random seed $r$ in the definition of replicability is one of the main distinctions between \Cref{def:general indist of outcomes} and \ref{def:replicability}. This shared random string can be seen as a way to
achieve a \emph{coupling} (see \Cref{def:coupling}) between two executions of the algorithm $A$. 
An interesting aspect of this definition is that replicability is verifiable; replicability under \Cref{def:replicability} can be
tested using polynomially many samples, random seeds $r$ and
queries to $A$.
We remark that the work of \cite{ghazi2021user} introduced the closely related 
notion of pseudo-global stability (see \Cref{def:pseudo-global stability}); 
the definitions of replicability and pseudo-global stability are equivalent up to polynomial factors in the parameters.

\paragraph{Differential Privacy.}
The notions of algorithmic indistinguishability and replicability that we have discussed so far have close connections with the classical definition of approximate differential privacy \cite{dwork2014algorithmic}. For $a,b, \eps,\delta \in [0,1]$, let $a \approx_{\eps, \delta} b$ denote the statement $a \leq e^\eps b + \delta$ and $b \leq e^\eps a + \delta$. We say that two probability distributions $P, Q$ are $(\eps, \delta)$-indistinguishable if $P(E) \approx_{\eps, \delta} Q(E)$ for any measurable event $E$. 
\begin{definition}
[Approximate Differential Privacy \cite{dwork2006our}]
\label{def:dp}
A learning rule $A$
is an $n$-sample $(\eps,\delta)$-differentially private if for any pair of samples $S, S' \in (\calX \times \{0,1\})^n$ that disagree on a single example, the induced posterior distributions $A(S)$ and $A(S')$ are $(\eps, \delta)$-indistinguishable.
\end{definition}
We remind the reader that, in the context of PAC learning, any hypothesis class $\calH$ can be PAC-learned by an approximate differentially-private algorithm if and only if it has a finite Littlestone dimension $\mathrm{Ldim}(\calH)$ (see \Cref{definition:littlestone-dimension}), i.e., there is a qualitative equivalence between online learnability and private PAC learnability \cite{alon2019private,bun2020equivalence,ghazi2021sample,alon2022private}.

{\paragraph{Broader Perspective.} Our work lies in the fundamental research direction of responsible ML. Basic concepts in this area, such as DP, replicability, and different forms of fairness, are formalized using various forms of stability. Therefore, it is natural and important to formally study the interrelations between different types of algorithmic stability. Our main purpose is to study statistical indistinguishability and replicability as properties of algorithms and, under the perspective of stability, investigate rigorous connections with DP. 
We view both replicability and DP as two fundamental blocks in the area of responsible and reliable ML. Hence, we believe that establishing formal connections between a priori not clearly related notions of “reliability” is a way to increase our understanding towards the design of responsible ML systems.}

\subsection{TV Indistinguishable Learning Rules}
\label{sec: tv rules}
As we discussed, our \Cref{def:general indist of outcomes} captures the property of a learning rule having
\emph{indistinguishable} outcomes under the resampling of its inputs from the same distribution.
In what follows, we instantiate 
\Cref{def:general indist of outcomes} with $d$ being the total variation (TV)    distance, probably the most well-studied notion of statistical distance in theoretical computer science. Total variation distance between two distributions $P$ and $Q$ over the probability space $(\Omega, \Sigma_\Omega)$ can be expressed as
\begin{align}
\label{eq:tv}
\begin{split}
    \dtv(P, Q)
&=
\sup_{A \in \Sigma_\Omega} P(A) - Q(A)\\
& = \inf_{(X,Y) \sim \Pi(P,Q)} \Pr[X \neq Y]\,,
\end{split}
\end{align}
where the infimum is over all couplings between $P$ and $Q$ so that the associated marginals are $P$ and $Q$ respectively. A \emph{coupling} between the distributions $P$ and $Q$ is a set of variables $(X,Y)$ on some common
probability space with the given marginals, i.e., $X \sim P$ and $Y \sim Q$. We think of a coupling as a construction of random variables $X,Y$ with prescribed laws. 

Setting $d = \dtv$ in \Cref{def:general indist of outcomes}, we get the following natural definition. For simplicity, we use the term TV indistinguishability to capture indistinguishability with respect to the TV distance.

\begin{definition}
[Total Variation Indistinguishability]
\label{def:TV stability}
A learning rule $A$ is $n$-sample $\rho$-$\TV$ indistinguishable if for any distribution over inputs $\calD$ and two independent sets $S, S' \sim \calD^n$ it holds that
\[
\E_{S,S' \sim \calD^n} [ \dtv( A(S), A(S') ) ] \leq \rho\,.
\]
\end{definition}


For some equivalent definitions, we refer to \Cref{sec:alternative-definitions}.
{Moreover, for some extensive discussion about the motivation of this definition, see \Cref{remark:motivation}.}
We emphasize that the notion of TV distance has very 
strong connections with statistical indistinguishability of distributions.
If two distributions $P$ and $Q$ are close in TV distance, then, intuitively, no statistical test
can distinguish whether an observation was drawn from $P$ or $Q$. 
In particular, if $\dtv(P,Q) = \rho$, then $\rho/2$ is the maximum advantage an analyst can achieve in determining whether a random sample $X$ came from $P$ or from $Q$ (where $P$ or $Q$
is used with probability 1/2 each).
In what follows, we focus on this particular notion
of statistical dissimilarity.

As a warmup, we start by proving a generalization
result for TV indistinguishable learners. Recall that if we \emph{fix}
some binary classifier we can show, using standard concentration bounds,
that its performance on a sample is close to its performance on the 
underlying population. However, when we train an ML algorithm using
a dataset $S$ to output a classifier $h$
we cannot just use the fact that it has small loss
on $S$ to claim that its loss on the population is small because
$h$ depends on $S$. The following result
shows that we can get such generalization bounds if $A$ is
a $\rho$-TV indistinguishable algorithm. We remark that a similar result regarding replicable algorithms
appears in \cite{impagliazzo2022reproducibility}. 
The formal proof, stated in a slightly 
more general way, is
in \Cref{app:gen}.

\begin{proposition}
[TV Indistinguishability Implies Generalization]\label{prop:ind implies generalization}
Let $\delta, \rho \in (0,1)^2$.
Let $\calD$ be a distribution over inputs
and $S = \{(x_i,y_i )\}_{i\in[n]}$ be a sample of size $n$ drawn i.i.d. from $\calD$. 
Let $h : \calX \to \{0,1\}$ be the output of an $n$-sample $\rho$-$\TV$ indistinguishable learning rule $A$ with input $S$. 
Then, with probability at least $1 - \delta - 4\sqrt{\rho}$ over $S$, 
it holds that,
\[
\left|\E_{h \sim A(S)}\left [L(h)\right] - \E_{h \sim A(S)}\left[\wh{L}(h) \right]\right|
\leq 
 \sqrt{\frac{\log(2/\delta)}{2n}} + \sqrt{\rho}\,,
\]
where $L(h) \triangleq \Pr_{(x,y) \sim \calD}[h(x) \neq y]$ and
$\wh{L}(h) \triangleq \frac{1}{n}\sum_{(x,y) \in S} 1\{h(x) \neq y\}.$
\end{proposition}


\subsection{Summary Of Contributions}
In this work, we investigate the connections between TV indistinguishability, replicability and differential privacy.


\begin{itemize}
    \item In \Cref{sec:ind and replicability}, we show that TV indistinguishability and replicability are equivalent. This equivalence holds for countable domains\footnote{We remark that the direction replicability implies TV indistinguishability holds for general domains.} and extends to general statistical tasks (cf. \Cref{app:general equivalence}). 

We remark that our transformations between replicable and TV indistinguishable learners do not change the (possibly randomized) input $\to$ output map which is induced by the learner; i.e., given a TV indistinguhishable learner $\mathcal{A}$, we transform it to a replicable learner $\mathcal{A}'$ such that $\mathcal{A}(S)$ and $\mathcal{A}'(S)$ are the same distributions over output hypotheses for every input sample $S$.

At this point we would like to highlight a subtle difference between replicability and other well studied notions of stability that arise in learning theory such as differential privacy, TV indistinguishability, one-way perfect generalization, and others.
The latter notions of stability depend only on the input $\to$ output map which is induced by the learner. In contrast, the definition of replicability has to do with the way the algorithm is implemented (in particular the way randomness is used). 
In other words, the definition of replicability enables having two learning rules $\mathcal{A}',\mathcal{A}''$ that compute exactly the same input $\to$ output map, but such that $\mathcal{A}'$ is replicable and $\mathcal{A}''$ is not. Thus, our equivalence suggests an interpretation of TV indistinguishability as an abstraction/extension of replicability that only depends on the input-output mechanism.

    

    \item In \Cref{sec:tv and dp}, we show that TV indistinguishability and $(\eps,\delta)$-DP are statistically equivalent. This equivalence holds for countable\footnote{We remark that the direction $(\eps, \delta)$-DP implies TV indistinguishability holds for general domains.} domains in the context of PAC learning. As an intermediate result, we also show that replicability and $(\eps,\delta)$-DP are statistically
    equivalent in the context of PAC learning, and this holds for general domains.
    \item In \Cref{sec:boosting}, we provide statistical amplification and boosting algorithms for TV indistinguishable learners for countable domains. En route, we improve the sample
    complexity of some routines provided in \cite{impagliazzo2022reproducibility}.
\end{itemize}

\subsection{Related Work}
\label{sec:related work}

Our work falls in the research agenda of replicable algorithm design, which was initiated by~\cite{impagliazzo2022reproducibility}.
In particular, \cite{impagliazzo2022reproducibility} introduced the notion of replicable learning algorithms, established that any statistical query algorithm can be made replicable, and designed replicable algorithms for various applications such as halfspace learning. Next, \cite{ahn2022reproducibility} studied reproducibility in optimization and \cite{esfandiari2022reproducible} provided replicable bandit algorithms.

The most closely related prior work to ours is the recent paper by \cite{bun2023stabilityIsStablest}.
In particular, as we discuss below in greater detail, an alternative proof of the equivalence between TV indistinguishability, replicability, and differential privacy follows from~\cite{bun2023stabilityIsStablest}. In contrast with our equivalence, the transformations by~\cite{bun2023stabilityIsStablest} are restricted to finite classes. 
On the other hand, \cite{bun2023stabilityIsStablest} give a constructive proof whereas our proof is purely information-theoretic. 

In more detail, \cite{bun2023stabilityIsStablest} establish a variety of equivalences between different notions of stability such as differential privacy, replicability, and one-way perfect generalization, and the latter contains TV indistinguishability as a special case:


\begin{definition}[(One-Way) Perfect Generalization~\cite{cummings2016adaptive,bassily2016typicality}]
    A learning rule $A:\calX^n \rightarrow \calY$ is $(\beta,\varepsilon,\delta)$-perfectly
    generalizing if, for every distribution $\calD$ over $\calX$, there exists a
    distribution $\calP_\calD$ such that, with probability at least $1-\beta$ over
    $S$ consisting of $n$ i.i.d. samples from $\calD$, and every set of outcomes 
    $\calO \subseteq \calY$
    \[
        e^{-\varepsilon}\left(\Pr_{\calP_\calD}[\calO] -\delta\right) \leq 
        \Pr[A(S) \in \calO] \leq e^\eps\Pr_{\calP_\calD}[\calO] +\delta \,.
    \]
    Moreover, $A$ is $(\beta, \eps, \delta)$-one-way perfectly generalizing if $
        \Pr[A(S) \in \calO] \leq e^\eps\Pr_{\calP_\calD}[\calO] +\delta$. 
\end{definition}
Note indeed that plugging $\eps=0$ to the definition of perfect generalization specializes the above definition to an equivalent variant of TV indistinguishability (see also \Cref{def:one-sided TV stability}).
\cite{bun2023stabilityIsStablest} derives an equivalence between replicability and one-way perfect generalization with $\eps>0$. However, in a personal communication they pointed out to us that their argument also applies to the case $\eps=0$, and hence to TV indistinguishability. In more detail, an intermediate step of their proof shows
that any $(\beta, \eps,\delta)$-perfectly generalizing algorithm $A$ is
also $(\beta, 0, 2\varepsilon + \delta)$-perfectly generalizing, which
is qualitatively equivalent with our main definition (see~\Cref{def:TV stability}).
As noted earlier our proof applies more generally to infinite countable domains but
is
non-constructive.

\paragraph{Differential Privacy.}  
Differential privacy \cite{dwork2008differential,dwork2010boosting,vadhan2017complexity,dwork2014algorithmic} is quite closely related to replicability.
The first connection between replicability and DP in the context of PAC learning was, implicitly, established by \cite{ghazi2021user} (for finite domains $\calX$), via the technique of correlated sampling (see \Cref{sec:coupling}) and the notion of pseudo-global stability (which is equivalent to replicability as noticed by \cite{impagliazzo2022reproducibility}):
\begin{definition}
[Pseudo-Global Stability \cite{ghazi2021user}]
\label{def:pseudo-global stability}
Let $\calR$ be a distribution over random strings.
A learning algorithm $A$ is said to be $n$-sample $(\eta, \nu)$-pseudo-globally stable if for any distribution $\calD$
there exists a hypothesis $h_r$ for every $r \in \mathrm{supp}(\calR)$ (depending on $\calD$) such
that
\[
    \Pr_{r \sim \calR} \left[\Pr_{S \sim \calD^n}[ A(S, r) = h_r ] \geq \eta \right] \geq \nu\,.
\] 
\end{definition}
The high-level connection between these notions appears to boil down to the notion of stability \cite{bousquet2002stability,poggio2004general,dwork2015preserving,abernethy2017online,bassily2016algorithmic,livni2020limitation} (see \cite{alon2022private} for further details between stability, online learnability and differential privacy). 
In particular, \cite{ghazi2021sample} showed that a class of finite Littlestone dimension admits a list-globally stable learner (see Theorem 18 in \cite{ghazi2021user}). The work of \cite{ghazi2021user} (among other things) showed (i) how to perform a reduction from list-global stability to pseudo-global stability via correlated sampling in finite domains (see Theorem 20 in \cite{ghazi2021user}) and (ii) how to perform a reduction from pseudo-global stability to approximate DP via DP selection (see Theorem 25 in \cite{ghazi2021user}).
We highlight that this equivalence between differential privacy
and replicability for finite domains was made formal by \cite{bun2023stabilityIsStablest} and was extended to arbitrary statistical tasks.

\paragraph{TV Stability.}
The definition of TV indistinguishability that we propose has close connections with the definition of TV stability.
This notion has appeared in the context of adaptive data analysis.
The work of \cite{bassily2016algorithmic} studied the following problem: suppose there is an unknown distribution $P$ and a set $S$ of $n$ independent samples drawn i.i.d. from $P$. The goal is to design an algorithm that, with input $S$, will accurately answer a sequence
of adaptively chosen queries about the unknown distribution $P$. The main question is how many samples must
one draw from the distribution, as a function of the type of queries, the number of queries, and the desired level of accuracy
to perform well? \cite{bassily2016algorithmic} provide various results that rely on the connections between algorithmic stability, differential privacy and generalization. To this end, they think of differential privacy as max-KL stability and study the performance of other notions of stability such as TV stability. Crucially, in their definition, TV stability considers any pair of neighboring datasets $S, S'$ and not two independent draws from $P$. More concretely, they propose
the following definition.
\begin{definition}
[Total Variation Stability \cite{bassily2016algorithmic}]
\label{def:old TV stability}
    A learning rule $A$ is $n$-sample $\rho$-$\TV$ stable if
    for any pair of samples $S, S' \in (\calX \times \{0,1\})^n$ that disagree on a single example, it holds that $
         \dtv(A(S),A(S')) \leq \rho$.
\end{definition}
We underline that for any constant $\rho$\footnote{In fact, even for $\rho \geq 1/n^c, 0 < c < 1.$} it is not challenging to obtain a $\rho$-TV stable 
algorithm in the learning setting we are interested in. It suffices
to just sub-sample a small enough subset of the data. Hence, any
class with finite VC dimension is TV stably learnable under this definition.
As it is evident from our results (cf. \Cref{prop:stab to dp}), this is in stark contrast with the definition we propose. We remind the readers that just
sub-sampling the dataset is not enough to achieve differential privacy. This is because
it is required that $\delta = o(1/n).$ We remark that the definition of total variation stability à la
 \cite{bassily2016algorithmic} also appears in \cite{raginsky2016information}.

The above definition of TV stability has close connections to machine unlearning. This problem refers to the ability of a user to delete their data that were used to train a ML algorithm. When this happens, the machine learning algorithm
has to move to a state as if it had never used that data for training, hence the
term \emph{machine unlearning}. One can see that \Cref{def:old TV stability} is 
suitable for this setting since it states that if one point of the dataset
is deleted, the distribution of the algorithm should not be affected very much. For
convex risk minimization problems, \cite{ullah2021machine} design TV stable algorithms based on noisy Stochastic
Gradient Descent (SGD). Such approaches lead to the design of efficient unlearning algorithms, which are based on sub-sampling the dataset and constructing a maximal coupling of Markov chains for the noisy
SGD procedure.

\paragraph{KL Stability and PAC-Bayes.}
In \Cref{sec:alternative-definitions} we provide some equivalent definitions to TV indistinguishability. In particular, \Cref{def:one-sided TV stability} has connections with the line of work that studies
distribution-dependent generalization bounds. To be more precise, if
instead of the TV distance we use the KL divergence
to measure the distance between the prior and the output of the algorithm
we get the definition of the quantity that is used to derive PAC-Bayes 
generalization bounds. Interestingly, \cite{livni2020limitation}
show that the PAC-Bayes framework cannot be used to derive distribution-free PAC learning bounds for
classes that have infinite Littlestone dimension; they show that for any algorithm that learns 1-dimensional linear classifiers (thresholds),
there exists a realizable distribution for which PAC-Bayes bounds are trivial. 
Recently, a similar PAC-Bayes framework was proposed in \cite{amit2022integral},
where the KL divergence
is replaced with a general family of Integral Probability Metrics (cf. \Cref{def:ipm}).

\paragraph{Probably Eventually Correct Learning.} The work of \cite{malliaris2022unstable} introduced
the \emph{Probably Eventually Correct} (PEC) model of
learning. In this model, a learner outputs the same hypothesis\footnote{Except maybe for a subset of $\calX$ that has measure zero under the data-generating distribution.}, with probability one, after a uniformly bounded number of revisions. Intuitively, this corresponds to the property that the global stability parameter is close to $1$. Interestingly, prior work on global stability \cite{bun2020equivalence,ghazi2021sample} had characterized \emph{Littlestone classes} as
being PAC learnable by an algorithm which outputs some fixed hypothesis with nonzero probability. However, the
frequency of this hypothesis was typically very small and its loss was a priori non-zero. \cite{malliaris2022unstable} give a new characterization to Littlestone classes by identifying them with the classes that can be PEC learned in a stable fashion. Informally, this means that
the learning rule for $\calH$ stabilizes on some hypothesis after changing
its mind at most $L$ times, where $L$ is the Littlestone dimension of $\calH$ (cf. \Cref{definition:littlestone-dimension}). Interestingly, \cite{malliaris2022unstable} manage to show 
that the well-known \emph{Standard Optimal Algorithm} (SOA) \cite{littlestone1988learning} is a stable PEC learner, using tools from the theory of universal learning \cite{bousquet2021theory,bousquet2022fine,kalavasis2022multiclass,hannekeuniversal}. Moreover, they list various
different notions of algorithmic stability and
show that they all have something in common: a class $\calH$ is learnable
by such learners if and only if its Littlestone dimension is finite.
Our main result shows that, indeed, classes that are learnable 
by TV indistinguishable learners fall into that category.

\section{TV Indistinguishability and Replicability}\label{sec:ind and replicability}

Our information-theoretic definition of TV indistinguishability
seems to put weaker restrictions on learning rules
than the notion of 
replicability in two ways: (i) it allows for \emph{arbitrary} couplings between the two executions
    of the algorithm (recall the coupling definition of TV distance, see Eq.\eqref{eq:tv}), and, (ii) it allows for \emph{different} couplings between every pair of datasets
    $S, S'$ (the optimal coupling in the definition of TV distance will depend on $S, S'$ of \Cref{def:TV stability}).
In short, our definition allows for \emph{arbitrary data-dependent} couplings, instead of just sharing the randomness 
across two executions. TV indistinguishability can be viewed as a statistical generalization of replicability (cf. \Cref{def:replicability}) since it describes a property of \emph{learning rules}
rather than \emph{learning algorithms}. 

In this section, we will show that TV indistinguishability and replicability are (perhaps surprisingly) equivalent in a rather
strong sense: under a mild measure-theoretic condition, every TV indistinguishable algorithm
can be converted into an \emph{equivalent} replicable one by \emph{re-interpreting}
its internal randomness.
This will be made formal shortly. 

We start by showing that any replicable algorithm is TV indistinguishable.

\begin{theorem}
[Replicability $\Rightarrow$ TV Indistinguishability]
\label{thm:replicability-implies-TV-ind}
\label{fact:replicability implies tv stability}
If a learning rule $A$ is $n$-sample $\rho$-replicable, 
then it is also $n$-sample $\rho$-$\TV$ indistinguishable.
\end{theorem}
\begin{proof}
Fix some distribution $\calD$ over inputs.
Let $A$ be $n$-sample $\rho$-replicable with respect to $\calD$. For the random variables $A(S), A(S')$ where $S, S' \sim \calD^n$ are two independent samples and using Eq.\eqref{eq:tv}, we have 
\begin{equation}
\label{eq:proof-eq1}
\E_{S,S' \sim \calD^n}[\dtv(A(S), A(S'))] =
\E_{S,S'\sim\calD^n}\left[\inf_{(h,h') \sim \Pi(A(S), A(S'))} \Pr[h \neq h']\right]\,.
\end{equation}
Let $\calR$ be the source of randomness that $A$ uses. The expected optimal coupling of Eq.\eqref{eq:proof-eq1} is at most $\E_{S,S'\sim \calD^n}\left[\Pr_{r \sim \calR} [A(S, r) \neq A(S', r)]\right]$. This
inequality follows from the fact that using shared randomness between
the two executions of $A$ is a particular way to couple the two random 
variables. To complete the proof, it suffices to notice that this upper bound is equal to
\[
\Pr_{S,S'\sim \calD^n, r \sim \calR}[A(S, r) \neq A(S', r)] \leq \rho\,.
\]
The last inequality follows since $A$ is $\rho$-replicable.
\end{proof}

We now deal with the opposite direction, i.e., 
we show that TV indistinguishability implies replicability. In order to be 
formal, we need to discuss some 
measure theoretic properties first.
Let us recall the definition of absolute continuity
for two measures. 
\begin{definition}[Absolute Continuity]
  Consider two measures $P,Q$ on a $\sigma$-algebra $\calB$ of subsets of $\Omega$.
We say that $P$ is absolutely continuous with respect to $Q$ if for any $E \in \calB$ such that $Q(E) = 0$, it holds that $P(E) = 0$. 
\end{definition}

Since the learning rules induce posterior distributions over hypotheses, this
definition extends naturally to such rules. 

\begin{definition}\label{def:absolutely cts learning rule for fixed D}
Given learning rule $A$, distribution over inputs $\calD$ and reference probability measure $\calP$, we say that $A$ is absolutely continuous with respect to $\calP$
on inputs from $\calD$ if, for almost every sample $S$ drawn from $\calD$, the posterior distribution $A(S)$ is absolutely continuous with respect to $\calP$. 
\end{definition}

In the previous definition, we fixed the data-generating distribution $\calD$. We next consider its distribution-free version.
\begin{definition}\label{def:absolutely cts learning rule for any D}
Given learning rule $A$ and reference probability measure $\calP$, we say that $A$ is absolutely continuous with respect to $\calP$ if, for any distribution over inputs $\calD$, $A$ is absolutely continuous with respect to $\calP$
on inputs from $\calD$. 
\end{definition}

If $\calX$ is finite, then one can take $\calP$ to be the uniform probability measure over $\{0,1\}^\calX$ and any learning rule is absolutely continuous with respect to $\calP$. 
We now show how we can find such a prior $\calP$ in the case where $\calX$ is 
countable.

\begin{claim}[Reference Probability Measure for Countable Domains]\label{clm:countable X reference measure}
    Let $\calX$ be a countable domain and $A$ be a learning rule. Then, there is 
    a reference probability measure $\calP$ such that $A$ is absolutely continuous with respect to $\calP$.
\end{claim}

\begin{proof}
    Since $\calX$ is countable, for a fixed $n$, 
    we can consider an enumeration of all the $n$-tuples $\{S_i\}_{i \in \nats}$. 
    Then, we can take $\calP$ to be a countable mixture of these probability measures, i.e.,
    $\calP = \sum_{i = 1}^\infty \frac{1}{2^i}A(S_i)$. Notice that since, each
    $A(S_i)$ is a measure and $1/2^i > 0$ for $i \in \nats,$ and,
    $\sum_{i=1}^\infty 1/2^i = 1$, we have
    that $\calP$ is indeed a probability measure. We now argue that each $A(S_i)$ is absolutely
    continuous with respect to $\calP$.
    Assume towards contradiction that this is not the case and let $E \in \calB$
    be a set such that $\calP(E) = 0$ but $A(S_j)(E) \neq 0$, for some $j \in \nats.$
    Notice that $A(S_j)$ appears with coefficient $1/2^j > 0$ in the mixture
    that we consider, hence if $A(S_j)(E) > 0 \implies 1/2^j A(S_j)(E) > 0.$ Moreover
    $A(S_i)(E) \geq 0, \forall i \in \nats,$ which means that $\calP(E) > 0$, 
    so we get a contradiction. 
\end{proof}



We next define when two learning rules $A, A'$ are equivalent.

\begin{definition}[Equivalent Learning Rules]\label{def:equivalent 
learning rules}
    Two learning rules $A, A'$ are equivalent if for every 
    sample $S$ it holds that $A(S) = A'(S)$, i.e., for the same input they induce   
    the same distribution over hypotheses.
\end{definition}

In the next result, we show that for every TV indistinguishable algorithm $A$, that is absolutely continuous with respect to some reference probability measure $\calP$, there exists an equivalent learning rule which is replicable. 

\begin{theorem}
[TV Indistinguishability $\Rightarrow$ Replicability]
\label{lem:TV stability to replicability}
Let $\calP$ be a reference probability measure over $\{0,1\}^\calX$,
and let $A$ be a learning rule that is $n$-sample $\rho$-$\TV$ indistinguishable 
and absolutely continuous with respect to $\calP$.
Then, there exists an equivalent learning rule $A'$ 
that is $n$-sample $\frac{2\rho}{1+\rho}$-replicable. 
\end{theorem}

In this section, we only provide a sketch of the proof and we refer the reader to \Cref{sec:proof of tv stability to replicability} for the complete one. Let us first
state how we can use the previous result
when $\calX$
is countable.

\begin{corollary}\label{cor:TV stability to replicability
for countable X}
    Let $\calX$ be a countable domain and let $A$
    be a learning rule that is $n$-sample $\rho$-$\TV$
    indistinguishable. 
    Then, there exists an equivalent learning rule $A'$ 
that is $n$-sample $\frac{2\rho}{1+\rho}$-replicable. 
\end{corollary}
The proof of this result follows immediately from \Cref{clm:countable X reference measure} and \Cref{lem:TV stability to replicability}.
\paragraph{Proof Sketch of \Cref{lem:TV stability to replicability}.} Let us consider a learning rule $A$ satisfying the conditions of \Cref{lem:TV stability to replicability}. Fix a distribution $\calD$ over inputs. 
The crux of the proof is that given two random variables $X, Y$ whose TV distance is bounded by $\rho$, we can couple them using only 
a carefully designed source of shared randomness $\calR$ so that the probability that the realizations of 
these random variables differ is at most $2\rho/(1+\rho).$  We can instantiate this observation with $X = A(S)$ and $Y = A(S')$. 
Crucially, in the countable $\calX$ setting, we can pick the shared randomness $\calR$ in a way that only depends on the learning rule $A$,
but not on $S$ or $S'$.
Let us now describe how this coupling works. 
Essentially, 
it can be thought of as a generalization of the von Neumann
rejection-based sampling which does not necessarily require that the distribution has bounded density.
Following~\cite{angel2019pairwise}, we pick $\calR$ to be a Poisson point process which generates points of the form $(h,y,t)$ with intensity\footnote{Roughly speaking, a point process is a (general) Poisson point
process with intensity $\lambda$ if (i) the number of points in a bounded Borel set $E$ is a Poisson random variable
with mean $\lambda(E)$ and (ii) the numbers of points in $n$ disjoint Borel sets forms $n$ independent random variables. For further details, we refer to \cite{last2017lectures}.} $\calP \times \mathrm{Leb} \times \mathrm{Leb}$, where $\calP$ is a reference probability measure with respect to which $A$ is absolutely continuous and $\mathrm{Leb}$ is the Lebesgue measure over $\reals_+$. Intuitively, $h \sim \calP$ lies in the hypotheses' space, $y$ is a non-negative real value and $t$ corresponds to a time value. 
The coupling mechanism performs \emph{rejection sampling} for each distribution we would like to couple (here $A(S)$ and $A(S'))$: it checks (in the ordering indicated by the time parameter) for each point $(h,y,t)$ whether $f(h) > y$ (i.e., if $y$ falls below the density curve $f$ at $h$) and accepts the first point that satisfies this condition. In the formal proof, there will be two density functions; $f$ (resp. $f'$) for the density function of $A(S)$ (resp. $A(S')$).
We also refer to \Cref{figure:poisson}. One can show (see \Cref{thm:pairwise opt coupling protocol}) that $\calR$ gives rise to a coupling between $A(S)$ and $A(S')$ under the condition that both measures are absolutely continuous with respect to the reference probability measure $\calP$.

This coupling technique appears in \cite{angel2019pairwise}.
We can then apply it and get
\[
    \Pr_{r \sim \calR}[A(S,r) \neq A(S',r)]
    \leq
    \frac{2\dtv(A(S),A(S'))}{1 + \dtv(A(S),A(S'))}\,.
\]
Taking the expectation with respect to the draws of $S,S'$,
we show (after some algebraic manipulations) that
$\Pr_{S,S'\sim \calD^n r \sim \calR}[A(S,r) \neq A(S',r)] \leq 2\rho/(1 + \rho)$. We conclude this section with the following remarks.

\begin{figure}[ht!]
    \centering
    \includegraphics[scale=0.8]{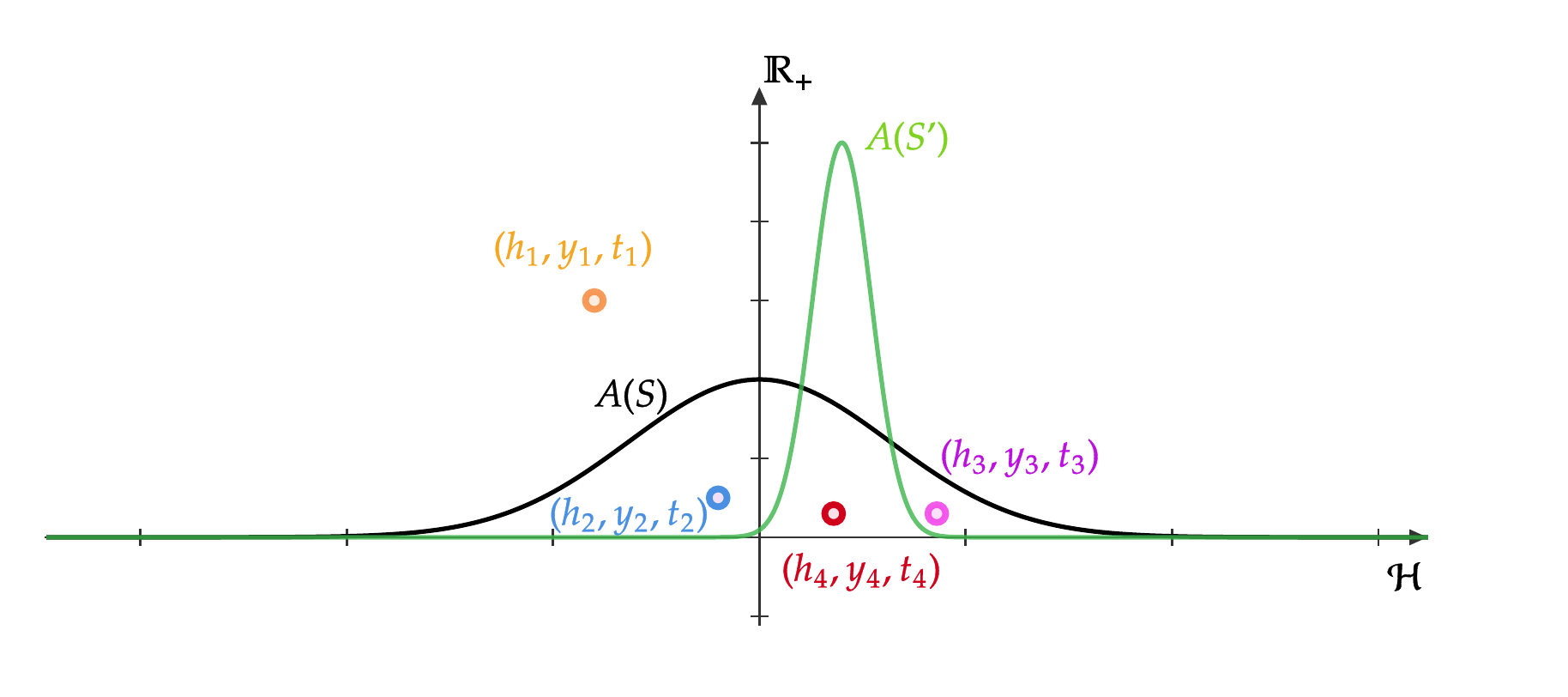}
    \caption{Our goal is to couple $A(S)$ with $A(S')$, where these two distributions are absolutely continuous with respect to the reference probability measure $\calP$. A sequence of points of the form  $(h, y, t)$ is generated by the Poisson point process with intensity $\calP \times \mathrm{Leb} \times \mathrm{Leb}$ where $h \sim \calP, (y,t) \in \reals^2_+$ and $\mathrm{Leb}$ is the Lebesgue measure over $\reals_+$ (note that we do not have upper bounds for the densities). Intuitively, $h$ lies in the hypotheses' space, $y$ is a non-negative real value and $t$ corresponds to a time value. 
    Let $f$ be the Radon-Nikodym derivate of $A(S)$ with respect to $\calP$.
    We assign the first (the one with minimum $t)$ value $h$ to $A(S)$ that satisfies the property that $f(h) > y$, i.e., $y$ falls below the density curve of $A(S)$. We assign a hypothesis to $A(S')$ in a similar manner. This procedure defines a data-independent way to couple the two random variables and naturally extends to multiple ones. In the figure's example, we set $A(S) = h_2$ and $A(S') = h_4$ given that $t_1 < t_2 < t_3 < t_4$. 
    }
    \label{figure:poisson}
\end{figure}


\begin{remark}[General Equivalence]
In \Cref{app:general equivalence}, we discuss how the above equivalence actually holds for general statistical tasks beyond binary classification. We first generalize the notions of indistinguishability, replicability and $\TV$ indistinguishability for general input spaces $\calI$ and output spaces $\calO$.
We then discuss that replicability and $\TV$ indistinguishability remain equivalent (under the same measure theoretic conditions) in these more general abstract learning scenarios.
\end{remark}

\begin{remark}[Implementation of the Coupling]
    We note that, in order to implement algorithm $A'$ of \Cref{lem:TV stability to replicability}, we need sample access to a Poisson point process with intensity $\calP \times \mathrm{Leb} \times \mathrm{Leb},$ where $\calP$ 
    is the reference probability measure from \Cref{clm:countable X reference measure} and $\mathrm{Leb}$ is the Lebesgue measure over $\reals_+$. 
    Importantly, $\calP$ depends only
    on $A$.
    Moreover, we need full access to the values of the density $f_i$ of the distribution $A(S_i)$ with respect to the reference probability measure $\calP$, for any sample $S_i$. 
    We underline that these quantities do not
    depend on the data-generating distribution
    $\calD$ (since we iterate over any possible sample).
\end{remark}

\begin{remark}[TV Indistinguishability vs.\ Replicability]
Notice that in the definition of replicability (cf. \Cref{def:replicability}) the 
source of randomness $\calR$ needs to be specified and by changing it we can observe
different behaviors for coupled executions of the algorithm. On the other hand,
the definition of $\TV$ indistinguishability (cf. \Cref{def:TV stability})
does not require the specification 
of $\calR$ as it states a property of the posterior distribution of 
the learning rule. 
\end{remark}

\section{TV Indistinguishability and Differential Privacy}
\label{sec:tv and dp}

In this section we investigate the connections between TV indistinguishability and approximate DP in binary classification.
Consider a hypothesis class $\calH \subseteq \{0,1\}^\calX$. We will say that $\calH$ is learnable by a $\rho$-TV indistinguishable learning rule $A$ if this rule satisfies
the notion of learnability under the standard realizable PAC learning model and is $\rho$-TV indistinguishable (see \Cref{def:learning-stable}).

The main result of this section is an equivalence between approximate DP and TV indistinguishability 
for countable domains $\calX$, in the context of PAC learning. We remark that
the equivalence of differential privacy with the notion of replicability is formally stated for finite outcome spaces (i.e., under the assumption that $\calX$ is finite) due to the use of a specific correlated sampling strategy for the direction that ``DP implies replicability'' in the context of classification \cite{ghazi2021user}.
Moreover, \cite{bun2023stabilityIsStablest} gave a constructive way to transform a DP algorithm to a replicable one for general statistical tasks and for finite domains. 
Thus, combining our results in \Cref{sec:ind and replicability}
and the result of \cite{ghazi2021user,impagliazzo2022reproducibility,bun2023stabilityIsStablest}, the equivalence of TV indistinguishability
and DP for \emph{finite} domains is immediate. We will elaborate
more on the differences of our approach and \cite{ghazi2021user, bun2023stabilityIsStablest}
later on. We also discuss our coupling and correlated sampling in \Cref{sec:coupling}.

Recall that a learner is $(\alpha,\beta)$-accurate if its misclassification probability is at most $\alpha$ with probability at least $1-\beta$.


\begin{theorem}
[$(\eps, \delta)$-DP $\Rightarrow$ TV Indistinguishability]
\label{prop:dp to stab}
Let $\calX$ be a (possibly infinite) domain and 
$\calH \subseteq \{0,1\}^\calX$. Let $\gamma \in (0,1/2),\alpha,\beta,\rho \in (0,1)^3 $.
Assume that $\calH$ is learnable by
an $n$-sample $(1/2-\gamma,1/2-\gamma)$-accurate $(0.1,1/(n^2\log(n)))$-differentially private learner. Then,  it is also learnable by an $(\alpha,\beta)$-accurate $\rho$-$\TV$ indistinguishable learning rule.
\end{theorem}

\paragraph{Proof Sketch of \Cref{prop:dp to stab}.} {The proof goes through the notion of global stability (cf.~\Cref{def:global-stability}). The existence 
of an $(\eps,\delta)$-DP learner implies that the hypothesis class $\calH$
has finite Littlestone dimension~\cite{alon2019private} (cf. \Cref{thm:dp learnable implies finite Littlestone}). Thus, we know
that there exists a $\rho$-globally stable learner for $\calH$~\cite{bun2020equivalence} (cf. \Cref{thm:finite littlestone
implies global stability}). The next step is to use the replicable heavy-hitters algorithm (cf.~\Cref{algo:replicable heavy hitters},~\cite{impagliazzo2022reproducibility}) with frequency parameter $O(\rho)$ and replicability parameter $O(\rho')$, where $\rho' \in (0,1)$ is the desired
TV indistinguishability parameter of the learning rule.
The global stability property implies that the list of heavy-hitters will be non-empty and it will contain at least one hypothesis with small error rate, with high probability. Finally, since
the list of heavy-hitters is finite and has bounded size,
we feed
the output into the replicable agnostic learner (cf.~\Cref{algo:replicable agnostic learner}). Thus, we have designed a replicable learner for $\calH$, and~\Cref{thm:replicability-implies-TV-ind} shows that this learner is also
TV indistinguishable. }

The formal proof of~\Cref{prop:dp to stab}
is deferred to~\Cref{sec:dp to tv}. We also
include a result which shows that \emph{list-global} stability
implies TV indistinguishability for general domains
and general statistical tasks, which could be of independent
interest (cf.~\Cref{lem:littlestone implies tv stable2}).

We proceed to the opposite direction where we provide an algorithm that takes as input a TV indistinguishable learning rule for $\calH$ and outputs a learner for $\calH$ which is $(\eps, \delta)$-DP.
In this direction countability of $\calX$ is crucial.

\begin{theorem}
[TV Indistinguishability $\Rightarrow$ $(\eps, \delta)$-DP]
\label{prop:stab to dp}
Let $\calX$ be a countable domain.
Assume that $\calH  \subseteq \{0,1\}^\calX$ is learnable by an $(\alpha, \beta)$-accurate $\rho$-$\TV$ indistinguishable learner $A$, for some $\rho \in (0,1), \alpha \in (0,1/2), \beta \in \left(0, \frac{1-\rho}{1+\rho}\right)$.
Then, for any $(\alpha',\beta',\varepsilon, \delta) \in (0,1)^4,$ it is also learnable by an $\left(\alpha + \alpha', \beta'\right)$-accurate $(\varepsilon, \delta)$-differentially private learner $A'$.
\end{theorem}
We refer to \Cref{sec:tv to dp} for the proof.
In the above statements, we omit the details about the sample complexity. We refer to \Cref{lem:littlestone implies tv stable2} and \Cref{lem:TV stable implies dp-learner} for these details.
Let us now comment on the differences between \cite{ghazi2021user, bun2023stabilityIsStablest} which establish a transformation from a replicable learner
to an approximately DP learner and our result. {The high-level
idea to obtain both of these results is similar. Essentially, the proof
of  \cite{ghazi2021user, bun2023stabilityIsStablest} can be viewed as a coupling
between sufficiently many posteriors of the replicable learning rule using \emph{shared randomness}
in order to achieve this coupling. In our proof, instead of using shared randomness
we use the reference measure we described in previous sections to achieve this coupling. 
We remark that we could have obtained the same qualitative result, i.e., that TV 
indistinguishability implies approximate DP, by using the transformation from replicability
to approximate DP of \cite{ghazi2021sample, bun2023stabilityIsStablest} in a black-box manner along with our result that TV indistinguishability implies replicability (cf. \Cref{lem:TV stability to replicability}). However, this leads to worse guarantees
in terms of the range of the parameters $\alpha,\beta,\delta,\eps,\rho$
than the ones stated in \Cref{prop:stab to dp}. Thus, we have chosen to do a more careful
analysis based on the coupling we proposed that leads to a stronger quantitative result.
More concretely,}  the proof in \cite{ghazi2021user, bun2023stabilityIsStablest} starts 
by sampling many random strings independently
of the dataset $\{S_i\}_{i \in [k]}$ and considers many executions of the algorithm
using the same random strings but different data.
In our algorithm we first
sample the sets $\{S_i\}_{i\in[k]}$
and then we consider an optimal coupling along 
the $\{A(S_i)\}_{i\in[k]}$ which is also independent of the dataset, thus it satisfies the DP 
requirements. 
Moreover, our procedure covers
a wider range of parameters $\alpha,\beta,\rho$ compared to 
\cite{ghazi2021user}.
The reason we need countability of $\calX$
is because it allows us to design a \emph{data-independent}
reference probability measure $\calP$, the same one as in \Cref{clm:countable X reference measure}.
Then, using this reference probability measure for the coupling
helps us establish the DP properties.
Nevertheless, we propose a simple change to our approach
which we conjecture 
applies to general domains $\calX$ and we leave it open as an interesting
future direction. For a more detailed discussion, we refer the reader
to \Cref{apx:dp beyond countable}. 

Interestingly, we underline that, as is shown
in~\cite{ghazi2021user,bun2023stabilityIsStablest} and as opposed to \Cref{prop:stab to dp}, replicability
implies DP in general spaces
(cf.~\Cref{thm:repl implies DP general tasks domains}).

We 
conclude this section by stating a general equivalence
between $(\eps,\delta)$-DP and replicability for PAC learning, that follows from the previous
discussion, in particular by combining~\Cref{thm:repl implies DP general tasks domains} \cite{ghazi2021user, bun2023stabilityIsStablest}, and \Cref{lem:dp implies replicability in general domains}.
\begin{theorem}[Replicability $\iff$ Differential Privacy in PAC Learning]\label{thm:DP PAC learning iff replicable}
    Let $\calX$ be a (possibly infinite) domain and let $\calH \subseteq \left\{0,1\right\}^\calX$. Then, $\calH$ is replicably learnable if and only if it is approximately-DP learnable.
\end{theorem}


{
\begin{remark}
[Dependence on the Parameters]
In the case of $\TV$ indistinguishability $\Rightarrow$ $\mathrm{DP}$, the blowup in the sample complexity is stated explicitly in \Cref{lem:TV stable implies dp-learner}.

For the direction $\mathrm{DP}$ $\Rightarrow$ $\TV$ indistinguishability it is a bit trickier to state the exact sample complexity blow-up because we do not make explicit use of the $\mathrm{DP}$ learner. Instead, we use the fact that the existence of a non-trivial $\mathrm{DP}$ learner implies that the class has finite Littlestone dimension and then we use an appropriate algorithm that is known to work for such classes. 
In this case, it suffices to let the parameters of the $\mathrm{DP}$ 
learner to be $\eps \in (0,0.1), \delta \in \left(0, \frac{1}{n^2\log(n)}\right), \alpha \in (0,1/2), \beta \in (0,1/2)$
and the parameters of the desired $\TV$ indistinguishable 
$(\alpha',\beta')$-accurate learner
are unconstrained, i.e., $\rho \in (0,1), \alpha' \in (0,1), \beta' \in (0,1)$. 
If we denote the Littlestone dimension of the class by $L$, then, as shown in \Cref{lem:littlestone implies tv stable2} the sample complexity of the $\TV$ indistinguishable learner is $\poly(L,1/\rho,1/\alpha',\log(1/\beta'))$\footnote{This holds under the (standard) assumption that uniform convergence holds for Littlestone classes. If this is not the case, we get $\poly(2^{2^L},1/\rho,1/\alpha',\log(1/\beta'))$ sample complexity (\Cref{cor:dp implies TV ind in general domains}).}. 
\end{remark}

\begin{remark}
[Beyond Binary Classification]
The only transformation that is restricted to binary classification is the one from $\mathrm{DP}$ to $\TV$ indistinguishability. All the other transformations, (and the boosting algorithms that we present in the upcoming section), extend to general statistical tasks.
Let us now shortly discuss how to extend our result e.g., to the multi-class setting, using results from the private multiclass learning literature 
\cite{jung2020equivalence,sivakumar2021multiclass}. \cite{jung2020equivalence} showed that private multiclass learnability implies finite multiclass Littlestone dimension and \cite{sivakumar2021multiclass} showed how to extend the binary list-globally stable learner that we use to the multiclass setting. Using these two main ingredients, the rest of our approach for the binary classification setting should extend to the multiclass setting. The extension to the regression problem seems to be more challenging. Even though \cite{jung2020equivalence} showed that private regression implies finiteness of some appropriate Littlestone dimension, it is not clear yet how to derive a (list-)globally stable algorithm for this problem. 
\end{remark}
}

\section{Amplifying and Boosting TV Indistinguishable Algorithms}
\label{sec:boosting}
In this section we study the following fundamental question.
\begin{question}
\label{quest}
Consider a weak $\TV$ indistinguishable learning rule both in terms of the indistinguishability parameter and the accuracy. Is it possible to amplify its indistinguishability and to boost its accuracy?
\end{question}

{For instance, in the context of approximate differential privacy, a series of works has lead to (constructive) algorithms that boost the accuracy and amplify the privacy guarantees (e.g., \cite{dwork2010boosting, bun2020equivalence, bun2023stabilityIsStablest}). This result builds upon the equivalence of online learnability and approximate differential privacy.}
Our result relating DP to TV indistinguishability implies the following existential result.

\begin{corollary}
\label{cor:boosting + amplification existential}
Let $\calX$ be a countable domain.
Suppose that for some sample size $n_0$, there exists an $(\alpha_0, \beta_0)$-accurate $\rho_0$-$\TV$ indistinguishable learner $A$ for a class $\calH \subseteq \{0,1\}^\calX$ with $\alpha_0 \in (0,1/2), \rho_0 \in (0,1), \beta_0 \in \left(0, \frac{1-\rho_0}{1+\rho_0}\right)$. Then, for any $(\alpha, \beta, \rho) \in (0,1)^3$, $\calH$ admits an $(\alpha, \beta)$-accurate $\rho$-$\TV$ indistinguishable learner $A'$.
\end{corollary}


{This result relies on connections between learnability by TV indistinguishable learners and finiteness of the Littlestone dimension of the underlying hypothesis class that were discussed in \Cref{sec:tv and dp}. In particular,
\Cref{cor:TV stable implies finite Littlestone dimension} shows that the existence of such a non-trivial TV indistinguishable learner implies that the $\calH$ has finite Littlestone dimension, and \Cref{lem:littlestone implies tv stable2}, states
that the finiteness of the Littlestone dimension of $\calH$ implies
the existence of an $(\alpha, \beta)$-accurate $\rho$-TV indistinguishable learner, for arbitrarily small choices of $\alpha, \beta, \rho$.}
It is not hard to see that we need to constrain $\alpha \in (0,1/2)$, because
the algorithm needs to have an advantage compared to the random classifier.
Moreover, it should be the case that $\beta \in (0,1-\rho).$ 
If $\beta \geq 1-\rho$ then the algorithm
which outputs a constant classifier
with probability $\beta$ and an $\alpha$-good
one with the remaining probability 
is $\rho$-TV indistinguishable and $(\alpha,\beta)$-accurate. An interesting open problem
is to investigate
what happens when $\beta \in \left( \frac{1-\rho}{1+\rho},1-\rho\right)$.

{We underline that \Cref{cor:boosting + amplification existential} is existential 
and does not make actual use of the weak TV indistinguishable learner that is given
as input. 
Hence, it is natural to try to come up with sample-efficient and constructive approaches that utilize the weak learner through black-box oracle calls to it
during the derivation of the strong one}.
In what follows, we aim to design such algorithms. 
We remind the reader that if we constrain ourselves to work in the setting
where $\calX$ is countable, then the absolute continuity requirement in the next theorems comes immediately, due to \Cref{clm:countable X reference measure}.

\paragraph{Indistinguishability Amplification.}
We first consider the amplification of
the indistinguishability guarantees of an algorithm. 
An important ingredient of our approach is a replicable algorithm
for finding heavy hitters of a distribution, i.e., elements whose frequency
is above some given threshold. This algorithm has appeared in \cite{ghazi2021user,impagliazzo2022reproducibility}. However, the dependence
of the number of samples in the confidence parameter in these works is polynomial.
We present a new variant of this algorithm that has polylogarithmic dependence on the
confidence parameter. Moreover, using a stronger concentration inequality, we improve
the dependence of the number of samples on the error parameter. We believe that this
result could be of independent interest. We also design an agnostic
learner for finite hypothesis classes. However, the dependence of
the number of samples on $|\calH|$ is polynomial. We believe that 
an interesting question is to design agnostic learners with
polylogarithmic dependence on $|\calH|$.
We refer the reader to \Cref{apx:amplification and boosting}.

\begin{theorem}
[Indistinguishability Amplification]
\label{lem:stability amplification}
Let $\calP$ be a reference probability measure over $\{0,1\}^\calX$ and $\calD$ be a distribution over inputs.
Consider the source of randomness $\calR$ to be a Poisson point process with intensity $\calP \times \mathrm{Leb} \times \mathrm{Leb},$ where
$\mathrm{Leb}$ is the Lebesgue measure over $\mathbb{R}_+$.
Consider a weak learning rule $A$ that is 
(i) $\rho$-$\TV$ indistinguishable with respect to $\calD$ for some $\rho \in (0,1)$,
(ii) $(\alpha,\beta)$-accurate for $\calD$ for some $(\alpha, \beta) \in (0,1)^2$, such that $\beta < \frac{2\rho}{\rho+1} - 2\sqrt{\frac{2\rho}{\rho+1}} + 1$,  and, (iii)
 absolutely continuous with respect to $\calP$ on inputs from $\calD$.
Then, for any $\rho', \eps, \beta' \in (0,1)^3$, there exists a learner $\textsc{Ampl}(A,\calR, \beta',\eps,\rho')$ that is $\rho'$-$\TV$ indistinguishable with respect to $\calD$, and $(\alpha+\eps,\beta')$-accurate for $\calD$. 
\end{theorem}
We remark that the above result makes strong use of the equivalence between replicability
and TV indistinguishability. Our algorithm is a variant of the amplification algorithm that
appeared in \cite{impagliazzo2022reproducibility}, which (i) works for a wider range of parameters 
and (ii) its sample complexity is polylogarithmic in the parameter $\beta'$.

\paragraph{Accuracy Boosting.} Next, we design an algorithm that boosts the accuracy of an $n$-sample $\rho$-TV indistinguishable algorithm and preserves its TV indistinguishability guarantee.
Our algorithm is a variant of the boosting mechanism provided in \cite{impagliazzo2022reproducibility}. Similarly as in the case
of amplification, our variant improves upon the dependence
of the number of samples
on the parameter $\beta'$.

\begin{theorem}
[Accuracy Boosting]\label{thm:boosting algorithm}
Let $\calP$ be a reference probability measure over $\{0,1\}^\calX$ and $\calD$ be a distribution over inputs.
Consider the source of randomness $\calR$ to be a Poisson point process with intensity $\calP \times \mathrm{Leb} \times \mathrm{Leb},$ where
$\mathrm{Leb}$ is the Lebesgue measure over $\reals_+$.
Consider a weak learning rule $A$ that is 
(i) $\rho$-$\TV$ indistinguishable with respect to $\calD$ for some $\rho \in (0,1)$,
(ii) $(1/2-\gamma,\beta)$-accurate for $\calD$ for some $(\gamma, \beta) \in (0,1)^2$, and,
(iii) absolutely continuous with respect to $\calP$ on inputs from $\calD$.
Then, for any $\beta', \eps, \rho' \in (0,1)^3$, there exists a learner $\textsc{Boost}(A,\calR, \eps)$ that is $\rho'$-$\TV$ indistinguishable with respect to $\calD$ and $(\eps,\beta')$-accurate for $\calD$.
\end{theorem}

We can combine the amplification and boosting results for a wide range of parameters and get the next corollary.
\begin{corollary}
    Let $\calX$ be a countable domain and $A$ be an $n$-sample $\rho$-$\TV$ indistinguishable $(\alpha,\beta)$-accurate algorithm, for some $\rho \in (0,1), \alpha \in (0,1/2), \beta \in \left(0, \frac{2\rho}{\rho+1} - 2\sqrt{\frac{2\rho}{\rho+1}}+1\right).$ Then, for any $\rho', \alpha', \beta' \in (0,1)^3,$ there exists a 
    $\rho'$-$\TV$ indistinguishable $(\alpha', \beta')$-accurate learner $A'$ that requires
    at most $O\left(\mathrm{poly}\left(1/\rho, 1/\alpha',\log(1/\beta')\right) \cdot n \right)$
    samples from $\calD$.
\end{corollary}
The proof of this result follows immediately from \Cref{lem:stability amplification}, \Cref{thm:boosting algorithm}, and from the fact that we can design the reference probability measure
$\calP$ for countable domains (cf. \Cref{clm:countable X reference measure}).
This result leads to two natural questions: what is the tightest range of $\beta$
for which we can amplify the stability parameter $\rho$ and
under what assumptions can we design such boosting and amplification algorithms
for general domains $\calX$? For a more detailed discussion, we refer the reader to \Cref{apx:boosting tight bound}, \Cref{apx:boosting beyond countable domains}.

{
\begin{remark}
[Dependence on the Parameters]
We underline that the polynomial dependence on $\rho$ in the boosting result is not an artifact of the algorithmic procedure or the analysis we provide, but it is rather an inherent obstacle in $\TV$ indistinguishability. \cite{impagliazzo2022reproducibility} show that in order to estimate the bias of a coin $\rho$-replicably with accuracy $\tau$ one needs at least $1/(\tau^2\rho^2)$ coin tosses. 
Since $\rho$-$\TV$ indistinguishability implies $(2\rho/(1+\rho))$-replicability as
we have shown (without any blow-up in the sample complexity), we also inherit this lower bound. Our main goal behind the study of the boosting algorithms is to identify the widest range of parameters $\alpha, \rho, \beta$ such that coming up with a $\rho$-$\TV$ indistinguishable algorithm switches from being trivial to being difficult. For example, in PAC learning we know that if the accuracy parameter is strictly less than $1/2$, then there are sample-efficient boosting algorithms that can drive it down to any $\eps > 0$. In the setting we are studying, it is crucial to understand the relationship between $\beta, \rho$, see \Cref{apx:boosting tight bound}.
\end{remark}}

\section{Conclusion}
\label{sec:privacy, replicability, stability}
In this work, we studied TV indistinguishability and established connections
to similar notions that have been proposed in the past, i.e., 
differential privacy, replicability, global stability, and pseudo-global
stability, under mild measure-theoretic assumptions
(e.g., countable $\calX$).
Our work leaves the following open problems:
\begin{enumerate}
    \item Does the equivalence between TV indistinguishability and replicability hold for general spaces, i.e., when the input domain is not countable?
    \item Does the equivalence between TV indistinguishability and $(\eps,\delta)$-DP hold for general spaces?
    \item How can we boost the correctness and amplify the indistinguishability parameter
    of a weak TV indistinguishable learner to a strong one in
    general spaces? 
    \item What is the minimal condition that characterizes TV indistinguishable PAC learnability? This is closely related to understanding the limits of TV indistinguishable boosting algorithms.
\end{enumerate}

\section{Acknowlegdements}
We thank Mark Bun, Marco Gaboardi, Max Hopkins, Russell Impagliazzo, Rex Lei, Toniann Pitassi,
Satchit Sivakumar, and Jessica Sorrell for illustrating discussions regarding the connection of this work with the recent paper~\cite{bun2023stabilityIsStablest}. We also thank 
Kyriakos Lotidis for helpful discussions about the Poisson point process.

\bibliography{bib}

\appendix

\section{Preliminaries and Additional Definitions}
\subsection{Preliminaries}

\paragraph{Probability Theory.} We first review some standard definitions from probability theory.

\begin{definition}
[Coupling]
\label{def:coupling}
A coupling of two probability distributions $P$ and
$Q$ is a pair of random variables $(X,Y)$, defined on the same probability space, such that the marginal distribution of $X$ is $P$ and the marginal distribution of $Y$ is $Q$.
\end{definition}

\begin{definition}
[Integral Probability Metric]
\label{def:ipm}
The Integral Probability Metric (IPM) between two probability measures $P$ and $Q$ over $\calO$ is defined as
\[
d_{\calF, \calO}(P,Q)
=
\sup_{f \in \calF}  
\left| \int_\calO f dP - \int_\calO f dQ  \right|
=
\sup_{f \in \calF}  
\left| \E_{x \sim P}[f(x)] - \E_{x \sim Q}[f(x)] \right|\,,
\]
where $\calF$ is a set of real-valued bounded functions $\calO \to \reals$.
\end{definition}
IPM distance measures are symmetric and non-negative. Note that the KL-divergence
is not a special case of IPM, rather it belongs to the family of $f$-divergences, that intersect with IPM
only at the TV distance. Such measures were recently used in order to derive PAC-Bayes style generalization bounds \cite{amit2022integral}. The definition of an $f$-divergence will not be useful in this work and we refer the interested reader to e.g., \cite{sason2016f}.

\paragraph{Learning Theory.} We next review some standard definitions in statistical learning theory. We start with the definition of the Littlestone dimension~\cite{littlestone1988learning}.
\begin{definition}[Littlestone Dimension \cite{littlestone1988learning}]
\label{definition:littlestone-dimension}
Consider a complete binary tree $T$ of depth $d+1$ whose internal
nodes are labeled by points in $\calX$ and edges by $\{0,1\}$, when they connect the parent to the right, left child, respectively. 
We say that $\calH \subseteq \{0,1\}^\calX$ Littlestone-shatters $T$ if for every root-to-leaf path $x_1,y_1,\ldots,x_d,y_d,x_{d+1}$ there exists some $h\in\calH$
such that $h(x_i) = y_i, 1 \leq i \leq d$. The Littlestone dimension is
denoted by $\mathrm{Ldim}(\calH)$ is defined to be the largest $d$ such
that $\calH$ Littlestone-shatters such a binary tree of depth $d+1$. If this happens for
every $d \in \nats$ we say that $\mathrm{Ldim}(\calH) = \infty$.
\end{definition}

We work under the well-known PAC learning model that was introduced in \cite{valiant1984theory}.
Let us denote the misclassification probability of a classifier $h$ by $\mathrm{err}_\calD(h) = \Pr_{(x,y) \sim \calD}[h(x) \neq y]$. Also, we say that $\calD$ is realizable
with respect to $\calH$ if there exists some $h^* \in \calH$ such that
$\mathrm{err}_\calD(h^*) = 0.$ Below, we slightly abuse notation and use the misclassification probability for distributions over classifiers.
\begin{definition}
[PAC Learnability \cite{valiant1984theory,shalev2014understanding}]
\label{def:pac}
An algorithm $A$ is $n$-sample $(\alpha,\beta)$-accurate for a hypothesis class $\calH \subseteq \{0,1\}^\calX$ if, for any realizable distribution $\calD$, it holds that
$
\Pr_{S \sim \calD^n}\left[
\mathrm{err}_\calD(A(S))
> \alpha\right] \leq \beta\,.
$
A hypothesis class $\calH$ is PAC learnable if, for any $\alpha, \beta \in (0,1)^2$,
there exist some $n_0(\alpha,\beta) \in \nats$ and an algorithm $A$ such that
$A$ is $n$-sample $(\alpha,\beta)$-accurate for $\calH,$ for any $n \geq n_0(\alpha,\beta).$
\end{definition}
For the purposes of this work, an algorithm $A$ should be thought of as a mapping from
samples to a \emph{distribution} over hypotheses. We want to design algorithms that
satisfy two desiderata: they are PAC learners for some given hypothesis class $\calH$
and they are total variation indistinguishable. In particular, we consider the following learning setting combining \Cref{def:TV stability} and \ref{def:pac}.
\begin{definition}
[Realizable Learnability by TV Indistinguishable Learner]
\label{def:learning-stable}
An algorithm $A$
is $n$-sample $(\alpha, \beta)$-accurate $\rho$-$\TV$ indistinguishable 
for a hypothesis class $\calH \subseteq \{0,1\}^\calX$
if, for any realizable distribution $\calD$, 
it holds that $(i)$ $A$ is $n$-sample $\rho$-$\TV$ indistinguishable and $(ii)$ $\Pr_{S \sim \calD^n}[\mathrm{err}_\calD(A(S)) > \alpha] \leq \beta$.
A hypothesis class $\calH$ is learnable by a $\TV$ indistinguishable algorithm if, for any $\alpha, \beta, \rho \in (0,1)$, there exist some $n_0(\alpha, \beta, \rho) \in \nats$ and an algorithm $A$ such that $A$ is $n$-sample $(\alpha, \beta)$-accurate $\rho$-$\TV$ indistinguishable for $\calH$ for any $n \geq n_0(\alpha, \beta, \rho)$.
\end{definition}
In the above definition, $n$ depends on $\alpha, \beta, \rho$ (and $\calH),$
but \emph{not} on the distribution.

\begin{definition}
[Uniform Convergence Property]
\label{def:uc}
We say that a domain $\calX$ and a class $\calH \subseteq \{0,1\}^\calX$
satisfy the uniform convergence property if
there exists a function $m^{\mathrm{UC}} : (0,1)^2 \to \nats$ such that
for any $\eps,\delta \in (0,1)$, and for every distribution $\calD$ over $\calX \times \{0,1\}$ it holds that if $S \sim \calD^m$ and $m \geq m^{\mathrm{UC}}(\eps, \delta)$, it holds that
$\sup_{h \in \calH} |L_S(h) - L_\calD(h)| \leq \eps$, with probability at least $1-\delta$, where $L_S$ (resp. $L_\calD)$ is the empirical (resp. population) loss.
\end{definition}

The fundamental theorem of learning theory \cite{vapnik2015uniform, blumer1989learnability} states
that the uniform convergence property is equivalent to the finiteness of the
VC dimension of $\calH.$ However, one needs to make some (standard) measurability
assumptions on $\calX, \calH$ to rule out pathological cases. For instance,
it is known that there classes with VC dimension $1$ where uniform convergence does
not hold \cite{ben20152}\footnote{We note that the proof of the existence of such a class holds under the continuum hypothesis.}. 
It is known that when $\calH$ is countable and has finite VC dimension uniform convergence holds \cite{bartlett2002rademacher}.

\subsection{General Definition of Indistinguishability}\label{apx:general definition 
ind}
While in the main body of the paper, we focused on binary classification, (most of) our proofs extend to general learning problems and so we first present a general abstract framework.

For general learning tasks, we can view learning rules (or algorithms) as randomized mappings
$A : \calI  \to \Delta_\calO$ which take as input instances from a domain
$\calI$ and map them to an element of the output space $\calO.$ 
We assume that there is a distribution
$\mu$ on $\calI$ that generates instances.

A second way to view the learning algorithm is via the mapping $A : \calI \times \calR \to \calO$.
Then $A$ takes as input an instance $I \sim \mu$ and a random string $r \sim \calR$ (we use $\calR$ for both the probability space and the distribution) corresponding to the algorithm's 
\emph{internal randomness} and outputs $A(I,r) \in \calO$. 
Thus, $A(I)$ is a distribution over $\calO$ whose randomness
comes from the random variable $r$, while $A(I,r)$ is a deterministic object.

The space $\Delta_\calO$ is 
endowed with some statistical dissimilarity measure.

\begin{definition}[Indistinguishability]\label{def:general ind for arbitrary tasks}
    Let $\calI$ be an input space, $\calO$ be an output space and $d$ be some statistical dissimilarity measure.
    A learning rule $A$ satisfies $\rho$-indistinguishability with respect to $d$
    if for any distribution $\mu$ 
    over $\calI$ and two independent instances $I,I' \sim \mu$, it holds that
    \[
        \E_{I, I' \sim \mu} [d\left(A(I ),A(I')\right)] \leq \rho \,.
    \]
\end{definition}

To illustrate the generality of our definition, we now show 
how we can instantiate $\calI, \calO, \mu, d$
to recover other definitions about stability of learning algorithms appearing in prior work.

\paragraph{Global Stability.} 
Global stability \cite{bun2020equivalence} is a fundamental property of learning algorithms that was recently used to establish an equivalence between online learnability and approximate differential privacy in binary classification. 
We show how we can recover
the definition of global stability.
Let us first recall the definition. 
\begin{definition}[Global Stability~\cite{bun2020equivalence}]
\label{def:global-stability}
Let $\calR$ be a distribution over random strings.
A learning rule A is $n$-sample $\eta$-globally stable if for any distribution $\calD$ there exists a hypothesis $h_\calD$ such
that 
\[
\Pr_{S \sim \calD^n, r \sim \calR}[A(S,r) = h_\calD ] \geq \eta\,.
\]
\end{definition}

In order to recover \Cref{def:global-stability} using
\Cref{def:general ind for arbitrary tasks} we let $(S,r) \in \calI, \mu = \calD^n \times \calR$ 
and $d(A(I,r),A(I',r')) = \mathbbm{1}_{A(I,r) \neq A(I',r')}.$ Thus, we have
that
\begin{align*}
    \E_{S, S' \sim \calD^n, r,r' \sim \calR} [\mathbbm{1}_{A(S,r) \neq A(S',r')}] \leq \rho \implies\\
    \Pr_{S, S' \sim \calD^n, r,r' \sim \calR} [A(S,r) \neq A(S',r')] \leq \rho.
\end{align*}

Notice that this gives us a two-sided version of the definition of global-stability. So far we have established that
$\Pr_{S, S' \sim \mu, r,r' \sim \calR} [A(S,r) = A(S',r')] \geq 1 - \rho > 0.$ Since two independent draws of the random variable
$A(S,r)$ are the same with non-zero probability it means that
it must have point masses. Moreover, there are countably
many such point masses. Let $\calH_m = \{h \in \calH: \Pr_{S\sim\calD^n, r \sim \calR}[A(S,r) = h]\}.$ Then,
\begin{align*}
    \Pr_{S, S' \sim \mu, r,r' \sim \calR} [A(S,r) = A(S',r')]
&=\sum_{h \in \calH_m} \left(\Pr_{S\sim\calD^n, r \sim \calR}[A(S,r) = h]\right)^2\\
&\leq \max_{h \in \calH_m}\Pr_{S\sim\calD^n, r \sim \calR}[A(S,r) = h] \cdot \sum_{h \in \calH_m} \Pr_{S\sim\calD^n, r \sim \calR}[A(S,r) = h]\\
&\leq \max_{h \in \calH_m}\Pr_{S\sim\calD^n, r \sim \calR}[A(S,r) = h]\\
&= \max_{h \in \calH}\Pr_{S\sim\calD^n, r \sim \calR}[A(S,r) = h]
\end{align*}
Thus, by chaining the two inequalities we have established,
we get that $\max_{h \in \calH_m}\Pr_{S\sim\calD^n, r \sim \calR}[A(S,r) = h] \geq 1-\rho$, so the algorithm $A$ satisfies 
the notion of global stability.


\subsection{Alternative Definitions of TV Indistinguishability}
\label{sec:alternative-definitions}

We now discuss alternative ways to define
TV indistinguishability. 

\subsubsection{TV Indistinguishability with Fixed Prior}
First, observe that the
definition we propose is two-sided in the sense that we require 
drawing two sets of i.i.d. samples. A different way to view TV indistinguishability
is by requiring that the output of the algorithm is close, in TV distance,
to some \emph{prior distribution}, which depends on the data-generating
process $\calD$ but is independent of the sample. Notice that we could introduce a similar one-sided general definition as a second viewpoint of \Cref{def:general indist of outcomes} (named Indistinguishability with Fixed Prior).
\begin{definition}
[TV Indistinguishability with Fixed Prior]
\label{def:one-sided TV stability}
A learning rule $A$ is $n$-sample $\rho$-fixed prior $\TV$ indistinguishable if for any distribution 
over inputs
$\calD$, there exists some prior $\calP_\calD$ 
such that for $S \sim \calD^n$ it holds that
\[
\E_{S \sim \calD^n} [ \dtv( A(S), \calP_\calD ) ] \leq \rho\,.
\]
\end{definition}
Notice that, using the triangle inequality, we can see
that this definition is equivalent to \Cref{def:TV stability}, up to a factor
of $2.$ Formally, we have the following result.
\begin{lemma}
\label{clm:two-sided TV stability iff one-sided TV stability}
If $A$ is $\rho$-$\TV$ indistinguishable 
then it is $\rho$-fixed prior $\TV$ indistinguishable. 
Conversely, if $A$ is $\rho$-fixed prior $\TV$ indistinguishable 
then it is $2\rho$-$\TV$ indistinguishable. 
\end{lemma}

We remark that if $A$ is TV indistinguishable with respect to a distribution over inputs $\calD$, one can show that it is also fixed prior TV indistinguishable with respect to $\calD$ where the fixed prior is equal to $\calP_\calD = \int_{S}A(S) d(\calD^n)$.

\begin{proof}
    For the first direction, we
    let $\calP_{S,S'}$ be a distribution with the property
    that $\dtv(A(S),\calP_{S,S'}) = \dtv(A(S'),\calP_{S,S'}) = \dtv(A(S),A(S'))/2$, e.g.,
    $\calP_{S,S'} = 1/2\cdot(A(S) + A(S'))$, for every $S, S'\sim \calD^n$. 
    We now define $\calP_\calD$ to be the average of $\calP_{S,S'}$ with respect to
    the measure of the product distribution of $S,S'.$ 
    We have that
    \begin{align*}
    \calP_\calD 
    &= \int_{S, S'} \calD^n(S) \calD^n(S') \frac{A(S) + A(S')}{2} dS dS' \\
    &= \int_T \left(\calD^n(T) 1\{S = T\} \frac{A(T)}{2} \left(\int_{S'}\calD^n(S')\right) + \calD^n(T) 1\{S' = T\}\frac{A(T)}{2} \left(\int_{S}\calD^n(S)\right) \right) dS dS'
    =\\
    &= \int_T \calD^n(T) A(T) dT\,. 
    \end{align*}
    This means that
    $\E_{S \sim \calD^n}[\dtv(A(S),\calP_\calD)] 
    = \int_S \calD^n(S) \dtv \left(A(S), \int_T \calD^n(T) A(T) dT\right) dS \leq \rho.$

    For the converse, notice that 
    \begin{align*}
        \E_{S,S'\sim \calD^n}[\dtv(A(S),A(S'))] &\leq 
    \E_{S,S'\sim \calD^n}[\dtv(A(S),\calP_\calD) + \dtv(A(S'),\calP_\calD)]\\
    &=
    \E_{S,S'\sim \calD^n}[\dtv(A(S),\calP_\calD)] + \E_{S,S'\sim \calD^n}[\dtv(A(S'),\calP_\calD)] \\
    &= 2 \E_{S\sim \calD^n}[\dtv(A(S),\calP_\calD)]\\
    &\leq 2\rho.
    \end{align*}
\end{proof}

\subsubsection{With High Probability TV Indistinguishability}
A different direction in which we can extend the definition of total variation indistinguishability has to do with
replacing the expectation with a high-probability style of bound. We remark
that \cite{impagliazzo2022reproducibility} provide a similar alternative 
definition in the context of their work.

\begin{definition}
[High-Probability TV Indistinguishability]
\label{def:one-sided TV stability high probability}
A learning rule $A$ is $n$-sample high-probability $(\eta,\nu)$-$\TV$ indistinguishable if for any distribution 
$\calD$ there exists some prior $\calP_\calD$
such that
\[
\Pr_{S \sim \calD^n} [ \dtv( A(S), \calP_\calD ) \leq \eta ] \geq 1 -\nu\,.
\]
\end{definition}
Notice that in the above definition we have used the fixed prior version of
TV indistinguishability to reduce the number of parameters, but it can also be stated
in its the two-sided version. It is not hard to see that the ``in expectation''
and the ``with high probability'' versions of the definition are 
qualitatively equivalent. Moreover, we can establish a quantitative connection
as follows. 

\begin{lemma}\label{clm:one-sided TV stable iff high-probability TV stable}
    If a learning rule $A$ is an $n$-sample $\rho$-fixed prior $\TV$ indistinguishable learner (cf. \Cref{def:one-sided TV stability}) then it is
    an $n$-sample high-probability
    $(\rho/\nu,\nu)$-$\TV$ indistinguishable learning rule (cf. \Cref{def:one-sided TV stability high probability}), for any $\rho \leq \nu < 1$. 
    Conversely, if a learnigng rule $A$ is an $n$-sample high-probability $(\eta, \nu)$-$\TV$ indistinguishable learner then 
    it is an $n$-sample $(\eta + \nu -\eta\cdot\nu)$-fixed prior $\TV$ indistinguishable learning rule.
\end{lemma}

\begin{proof}
    The proof of the first part of claim is a direct consequence of Markov's inequality. 
    Notice that $\dtv(A(S), \calP_\calD)$ is random variable whose expected value
    is bounded by $\rho$. Thus, we have that
    \[
    \Pr_{S \sim \calD^n} [\dtv(A(S), \calP_\calD) \geq \rho/\nu ] \leq \nu \,.
    \]
    Hence, we can see that $A$ is a high-probability $(\rho/\nu, \nu)$-TV indistinguishable learning rule. 

    We now move to the second part of the claim. Let $\calE$ be the event
    that $\dtv(A(S),\calP_\calD) \geq \eta.$ Then, we have that
    \begin{align*}
            \E_{S \sim \calD^n} [\dtv(A(S), \calP_\calD)] &= 
            \E_{S \sim \calD^n} [\dtv(A(S), \calP_\calD) | \calE]\Pr[\calE] + 
            \E_{S \sim \calD^n} [\dtv(A(S), \calP_\calD) | \calE^c]\Pr[\calE^c] \\
            &\leq 1 \cdot \nu + \eta \cdot (1-\nu) \\
            &= \eta + \nu -\eta\cdot\nu.
    \end{align*}
\end{proof}

\subsection{Coupling and Correlated Sampling} 
\label{sec:coupling}
Coupling is a fundamental notion in probability theory with many applications \cite{levin2017markov}. The correlated sampling problem, which has applications in various domains, e.g., in sketching and approximation algorithms \cite{broder1997resemblance,charikar2002similarity}, is described in \cite{bavarian2016optimality} as follows: Alice and Bob are given probability distributions $P$ and $Q$, respectively, over a finite set $\Omega$. \emph{Without any communication, using only shared
randomness} as the means to coordinate, Alice is required to output an element $x$ distributed according to $P$ and Bob is required to output an element $y$ distributed according to $Q$. Their goal is to minimize the disagreement probability $\Pr[x \neq y]$, which is comparable with $\dtv(P,Q)$. Formally, a correlated sampling strategy for a finite set $\Omega$ with error $\eps : [0,1] \to [0,1]$ is specified by a probability space $\calR$ and a pair of functions $f,g : \Delta_\Omega \times \calR \to \Omega$, which are measurable in their second argument, such that for any pair $P,Q \in \Delta_\Omega$ with $\dtv(P,Q) \leq \delta$, it holds that (i) the push-forward measure $\{f(P,r)\}_{r \sim \calR}$ (resp. $\{g(Q, r)\}_{r \sim \calR}$) is $P$ (resp. $Q$) and (ii)
$\Pr_{r \sim \calR}[f(P,r) \neq g(Q,r)] \leq \eps(\delta)$.
We underline that a correlated sampling strategy is \emph{not} the same as a coupling,
in the sense that the latter requires a single function $h : \Delta_\Omega \times \Delta_\Omega \to \Delta_{\Omega \times \Omega}$ such that for any $P,Q$, the marginals of $h(P,Q)$ are $P$ and $Q$ respectively.
It is known that for any coupling function $h$, it holds that $\Pr_{(x,y) \sim h(P,Q)}[x \neq y] \geq \dtv(P,Q)$ and that this bound is attainable. Since $\{(f(P,r), g(Q,r))\}_{r \sim \calR}$ induces a coupling, it holds that $\eps(\delta) \geq \delta$ and, perhaps surprisingly, there exists a strategy with $\eps(\delta) \leq \frac{2\delta}{1+\delta}$ \cite{broder1997resemblance,kleinberg2002approximation,holenstein2007parallel} and this result is tight \cite{bavarian2016optimality}.
A second difference between coupling and correlated sampling has to do with the size of $\Omega$: while correlated sampling strategies can be extended to infinite spaces $\Omega,$ it remains open whether there exists a correlated sampling strategy
for general measure spaces $(\Omega, \calF, \mu)$ with any non-trivial error bound \cite{bavarian2016optimality}.
On the other hand, coupling applies to spaces $\Omega$ of any size.

\cite{ghazi2021user} studied user-level privacy and introduced the notion of pseudo-global stability, which is essentially the same as replicability as observed by \cite{impagliazzo2022reproducibility}. \cite{ghazi2021user} showed that pseudo-global stability is qualitatively equivalent to approximate differential privacy. Their main technique was the use of correlated sampling that allowed users to output the same learned hypothesis (stability) employing shared
randomness.
We mention that \cite{ghazi2021user} provide their results for finite outcome space (i.e., $\calX$ is finite and thus $\calH \subseteq \{0,1\}^\calX$ is too).
{In particular, they need finiteness of the domain in order to apply correlated sampling which is used during their ``DP implies pseudo-global stability'' reduction.}
They mention that their results can be extended to the case where $\calX$ is infinite and that this does require non-trivial generalization of tools such as correlated sampling and some measure-theoretic details to that setting\footnote{To be more specific, the proof of Theorem 20 in \cite{ghazi2021user} requires to define the correlated sampling strategy over the space $2^\calX$ a priori (independently of the observed samples and input algorithm). Hence while the strategy is applied to distributions with finite support, an extension to infinite domain in that proof would require some modifications.}; we refer to a discussion in Section 5.3 of \cite{bavarian2016optimality} about the assumptions needed in order to achieve
correlated sampling in infinite spaces. 
{Similarly, the last step of the constructive transformation of a DP algorithm to a replicable one provided in \cite{bun2023stabilityIsStablest} uses correlated sampling and is hence also given for finite domains.}
For further comparisons between our coupling and the correlated sampling problem of \cite{bavarian2016optimality}, we refer to the discussion in \cite{angel2019pairwise} after Corollary 4.

A very useful tool for our derivations is a coupling protocol that can be found
in \cite{angel2019pairwise}.

\begin{theorem}[Pairwise Optimal Coupling \cite{angel2019pairwise}]\label{thm:pairwise opt coupling protocol}
    Let $\calS$ be any collection of random variables that are absolutely continuous with
    respect to a common probability measure\footnote{This result extends to the setting where $\mu$
    is a $\sigma$-finite measure, but it is not needed for the purposes of our work.}
    $\mu.$ Then, there exists a coupling of the
    variables in $\calS$ such that, for any $X, Y \in \calS$,
    \[
        \Pr[X \neq Y] \leq \frac{2\dtv(X,Y)}{1+ \dtv(X,Y)} \,.
    \]
    Moreover, this coupling requires sample access to a Poisson point
    process with intensity $\mu \times \mathrm{Leb} \times \mathrm{Leb}$, where
    $\mathrm{Leb}$ is the Lebesgue measure over $\mathbbm{R}_+,$ and full access
    to the densities of all the random variables in $\calS$ with respect to $\mu.$
\end{theorem}

An intuitive illustration of how it works can be found in \Cref{figure:poisson}.

\subsection{Discussion on \Cref{def:TV stability}}
\label{remark:motivation}
{
We discuss more extensively the TV Indistinguishability definition.
One important motivation for the definition of TV indistinguishability is to show that replicability can be equivalently defined using the same high-level template like the well-studied PAC-Bayes framework, where one shows that the outputs of the algorithms are close, under the KL divergence, with some data-independent priors. In other words, our results show how to organize and view different well-studied notions of stability using the same template.

Moreover, an interpretation of the replicability definition is that two executions of the algorithm over independent datasets should be coupled using just shared internal randomness. However, this is one of potentially infinite ways to couple the two executions. Our definition, which we find quite natural, captures exactly this observation and allows for general couplings between two random runs. It is also worth noting that, to the best of our knowledge, all the notions of algorithmic stability that have been proposed in the past do not depend on the source of internal randomness of the algorithm. However, this is not the case with replicability.

Let us now present a concrete algorithm whose stability property is easier to prove under the new definition. \cite{ghazi2021user} presented a procedure that transforms a list-globally stable algorithm to a replicable one (Algorithm 1, page 9 in \cite{ghazi2021user}). Crucially, in the last step of this algorithm the authors use a correlated sampling procedure to prove the replicability property. This procedure induces a computational overhead to the overall algorithm, and it is not clear even if it is computable beyond finite domains. On the other hand, the TV indistinguishability property is immediate. Thus, the transformation from list-global stability to TV indistinguishability is computationally efficient and holds for general domains whereas the transformation from list-global stability to replicability is not.

To the best of our knowledge, most of the replicable algorithms that have been developed use their internal randomness over data-independent distributions. To make this point more clear let us consider the replicable SQ oracle of \cite{impagliazzo2022reproducibility}. In this work, the authors use randomness over distributions that are independent of the input sample $S$. Thus, no matter how the internal randomness is implemented, when one shares it across two executions the internal random choices of the algorithm are the same.

However, there are algorithms, like Algorithm 1 in \cite{ghazi2021user}, that use internal randomness over a data-dependent distribution. If the algorithm makes random choices over data-dependent quantities like in \cite{ghazi2021user}, when one shares the randomness across two executions the internal random choices are not necessarily the same even if the TV distance between the two distributions is small, unless one specifies carefully the source of internal randomness (i.e., using some coupling). This can lead to significant computational overhead when the domain is finite, computability issues when the domain is countable, and for general domains it is not clear yet that going from TV indistinguishability to replicability is possible. Hence, one advantage of TV indistinguishability is that it provides a relaxation over the stronger definition of replicability, which is the notion that our definition builds upon.
}

\section{Useful Replicable Subroutines}
\label{apx:replicability tools}
In this section we present various replicable subroutines that
will be useful in the derivation of our results.
\subsection{Replicability Preliminaries}\label{apx:replicability preliminaries}
Recall the Statistical Query (SQ) model that was introduced by \cite{kearns1998efficient} and is a restriction of the PAC 
learning model, appearing in various learning theory contexts \cite{blum2003noise,gupta2011privately,chen2020classification,goel2020statistical,fotakis2021efficient}. In the SQ model, the learner interacts with an oracle
in the following way: the learner submits a statistical query to the oracle
and the oracle returns its expected value, after adding some noise to it. More
formally, we have the following definition.
\begin{definition}
[\cite{kearns1998efficient}]
\label{Statistical Query Oracle}
    Let $\tau,\delta \in (0,1)^2,\calD$ be a distribution over the domain $\calX$ and
    $\phi: \calX \rightarrow [0,1]$ be a query. Let $S$ be an i.i.d. sample of
    size $n = n(\tau,\delta).$ Then, the statistical query oracle outputs 
    a value $v$ such that $|v - \E_{x \sim \calD}[\phi(x)]| \leq \tau,$
    with probability at least $1-\delta.$
\end{definition}

Essentially, using a large enough number of samples, the SQ oracle returns
an approximation of the expected value of a statistical query whose range is bounded. \cite{impagliazzo2022reproducibility} provide a replicable implementation of an SQ 
oracle with a mild blow-up in the sample complexity.
\begin{theorem}[Replicable SQ Learner \cite{impagliazzo2022reproducibility}]
\label{thm:replicable sq learner}
    Let $\tau,\delta,\rho \in (0,1)^3, \delta \leq \rho/3, \calD$ be a distribution over some domain $\calX$, and
    $\phi: \calX \rightarrow [0,1]$ be a query. Let $S$ be an i.i.d. sample of
    size
    \[
    n = O\left(\frac{1}{\tau^2\rho^2} \log(1/\delta)\right) \,.
    \]
    Then there exists a $\rho$-replicable SQ oracle for $\phi.$
\end{theorem}
The interpretation of the previous theorem is that we can estimate replicably
statistical queries whose range is bounded.

The following result that was proved in \cite{impagliazzo2022reproducibility}
is useful for our derivations.
\begin{claim}[$\rho$-Replicability $\implies (\eta,\nu)$-Replicability \cite{impagliazzo2022reproducibility}]\label{clm:repl implies pseudo-global stability}
    Let $A$ be a $\rho$-replicable algorithm and $\calR$ be its source of randomness. Then for any $\nu \in [\rho,1)$, it holds that
    \[
        \Pr_{r\sim \calR}\left[\left\{\exists h \in \calH: \Pr_{S\sim \calD^n} [A(S,r) = h] \geq 1-\frac{\rho}{\nu}\right\}\right] \geq 1-\nu \,.
    \]
\end{claim}

{Notice that in the definition of replicability (\Cref{def:replicability}), the learner
shares all the internal random bits across its two executions.
A natural extension is to consider learners that share
only \emph{part} of their random bits, i.e., they have access to private random bits that are not shared
across its executions and public random bits that are shared. A result 
in~\cite{impagliazzo2022reproducibility} shows that these learners are, essentially, 
equivalent to the ones that use only private bits. To be more precise, we say that
a learner $A$ is $\rho$-replicable with respect to $r_{pub}$ if
\[
    \Pr_{S,S'\sim\calD^n, r_{priv}, r_{priv}', r_{pub} \sim \calR}[A(S,r_{priv},r_{pub}) = A(S',r'_{priv},r_{pub})] \geq 1-\rho \,.
\]
The following result states this property formally.
\begin{lemma}[Public, Private Replicability $\implies$ Replicability \cite{impagliazzo2022reproducibility}]\label{lem:private replicability}
Let $A$ be an $n$-sample $\rho$-replicable learner with respect to $r_{pub}$. Then, $A$
is a $n$-sample $\rho$-replicable learner with respect to $(r_{pub},r_{priv}).$
\end{lemma}
This result allows us to think of a replicable learner as having access to two different
sources of randomness, one that is private to its execution and one that is shared across
the executions. We will make use of it in transformations from
DP learners to replicable learners and some boosting results.}

\subsection{Replicable Heavy-Hitters}
In the analysis of the replicable heavy-hitter algorithm (cf. \Cref{algo:replicable heavy hitters})
we will use the Bretagnolle-Huber-Carol inequality that bounds the estimation error
of the parameters of a multinomial distribution from samples.

\begin{lemma}[Bretagnolle-Huber-Carol Inequality \cite{vaart1997weak}]
\label{lem:bretagnolle inequality}
    Let $p = (p_1,\ldots,p_k)$ multinomial distribution supported on $k$ elements.
    Then, given access to $n$ i.i.d. samples from $p$ we have that
    \[
    \Pr\left[\sum_{i=1}^k |\hat{p}_i - p_i| \geq \varepsilon\right] \leq 2^k e^{-n\varepsilon^2/2}\,,
    \]
    for every $\varepsilon \in (0,1),$ where $\hat{p}_i$ is the empirical frequency of
    item $i$ in the sample $S.$
\end{lemma}

The replicable heavy-hitters algorithm is depicted in \Cref{algo:replicable heavy hitters}. As we alluded before,
this approach is very similar to \cite{ghazi2021user, impagliazzo2022reproducibility}.
However, in our approach we treat the confidence parameter and the reproducibility 
parameters differently. Moreover, since we make use of \Cref{lem:bretagnolle inequality},
we are able to reduce the sample complexity of the algorithm.

\begin{algorithm}[ht!]
\caption{Replicable Heavy-Hitters}
\label{algo:replicable heavy hitters}
\begin{algorithmic}[1]
\State \texttt{Input: Sample access to a distribution $\calD$ over some domain $\calX$}
\State \texttt{Parameters: Threshold $v$,  error $\eps$, confidence $\delta$, replicability $\rho$}
\State \texttt{Output: List of elements $L$ in $\calX$}
\State $n_1 \gets \frac{\log\left(2/(\min\{\delta,\rho\}(v-\varepsilon))\right)}{v-\varepsilon}$ \State $S_1 \gets n_1$ i.i.d. samples.
 from $\calD$
\State $\calX_h \gets $ unique elements of $S_1$   \Comment{Notice that $|\calX_h| \leq n_1.$} \label{step:unique elements}
\State $n_2 \gets \frac{32\left(\ln(2/\min\{\delta,\rho\}) + |\calX| + 1\right) }{\rho^2\eps^2}$
\State $S_2 \gets n_2$ i.i.d. samples from $\calD$
\State $\hat{p}_x \gets \mathrm{freq}_S(x), \forall x \in \mathcal{X}_h$ \Comment{$\hat{p}_x$ is the empirical frequency of every potential heavy hitter} 
\State $v' \gets U[v-\varepsilon/2, v + \varepsilon/2]$ \Comment{Set the threshold for acceptance of a heavy-hitter.}
\State $L \gets \{x \in \calX_h: \hat{p}_x \geq v'\}$ \Comment{Drop the elements of $\calX_h$
that fall below the threshold.}
\State Output $L$ 
\end{algorithmic}
\end{algorithm}

\begin{lemma}
\label{lem:replicable heavy hitters}
Let $\calD$ be distribution supported on some domain $\calX$
and denote by $\calD(x)$ the mass 
that it puts on $x \in \calX$.
For any $\eps, \delta, \rho, v \in (0,1)^4$ such that
$(v-\eps, v + \eps) \subseteq (0,1)$, \Cref{algo:replicable heavy hitters} is $\rho$-replicable and
outputs a list $L$ such that, with probability $1-\delta$, for all $x \in \calX$:
\begin{itemize}
    \item If $\calD(x) < v - \eps$ then $x \notin L.$
    \item If $\calD(x) > v + \eps$ then $x \in L.$ 
\end{itemize}
Its sample complexity is at most
$O\left( \frac{\log(1/(\min\{\delta,\rho\}(v-\varepsilon)))}{(v-\varepsilon)\rho^2\varepsilon^2}\right).$
\end{lemma}

\begin{proof}
    We first prove the correctness of the algorithm with 
    the desired accuracy $\varepsilon$ and confidence $\delta$. For
    simplicity, let us assume that $\delta \leq \rho/4.$ Otherwise, we can simply set $\delta = \rho/4$.
    After we pick $n_1$ points, the probability that a $(v-\varepsilon)$-heavy-hitter
    of the distribution is not included in $S_1$ is at most
    \[
        (1-(v-\varepsilon))^{n_1} \leq e^{-(v-\varepsilon)\cdot n_1} \leq \frac{\delta\cdot(v-\varepsilon)}{2} \,.
    \]

    Since there are at most $1/(v-\varepsilon)$ such heavy-hitters, we can see
    that with probability at least $\delta/2$ all of the are included in $S_1.$ Let
    us call this event $\calE_1$ and condition on it for the rest of the proof.

    Let us consider a distribution $\widehat{\calD}$ that puts the same mass on every
    element of $\calX_h$ as $\calD$ and the remaining mass on a new special element $e.$
    We can sample from $\widehat{\calD}$ in the following way: we draw a sample from 
    $\calD$ and if it falls in $\calX_h$ we return it, otherwise we return $e$. 
    Thus, we can see that if we draw $n$ samples from $\widehat{\calD}$, they are
    distributed according to a multinomial distribution supported on
    $\calX_h \cup \{e\}$. Thus, \Cref{lem:bretagnolle inequality} applies to this setting
    which means that if we draw $n_2$ i.i.d. samples from $\widehat{\calD}$ we have that
    \[
        \Pr\left[\sum_{i=1}^k |\hat{p}_i - p_i| \geq \frac{\varepsilon \rho}{4}\right] \leq 2^k e^{-n_2\varepsilon^2\rho^2/32} \leq e^k e^{-n_2\varepsilon^2\rho^2/32} = e^{k - n_2\varepsilon^2\rho^2/32}\,,
    \]
    where $k = |\calX_h| + 1.$ Thus, $e^{k - n_2\varepsilon^2\rho^2/32} = e^{-\ln(2/\delta)} \leq \frac{\delta}{2}.$ We call this event $\calE_2$ and condition on it for the rest
    of the proof. Notice that under this event we have that $|\hat{p}_x - p_x| \leq \frac{\eps\rho}{4} < \frac{\eps}{2}, \forall x \in \calX_h.$ Since $v' \geq v - \eps/2$ it means
    that if $\hat{p}_x \geq v' \geq v- \eps/2\implies 
    p_x + \eps/2 > v - \eps/2 \implies p_x > v - \eps.$ Similarly,
    we get that if $\hat{p}_x < v' \implies p_x < v + \eps.$
    Hence, we see that the algorithm is correct with probability at least $1-\delta/2-\delta/2 = 1-\delta.$ This concludes 
    the correctness proof.

    We now focus on the replicability of the algorithm. Let $\calX_h^1$ be the unique elements
    at \Cref{step:unique elements} of the algorithm in the first run and $\calX_h^2$ in the second run. 
    Notice that if $x \in (\calX_h^1 \setminus \calX_h^2) \cup (\calX_h^2 \setminus \calX_h^1)$
    then, with probability at least $1-\delta/2-\delta/2 = 1-\delta$,
    the element $x$ is not a $(v-\eps)$-heavy-hitter, so, with probability at least $1-\delta/2$, it will not be included in the output of the execution that it appears in. 
        Let $E = \calX_h^1 \cap \calX_h^2$ and denote by $L_1, L_2,$ the outputs of the first, second 
    execution, respectively. We need to bound the probability of the event $\calE = \{\exists x \in E: x \in L_1\setminus L_2 \cup L_2\setminus L_1\}.$ Let $\hat{p}_x^1, \hat{p}_x^2$
    the empirical frequencies of $x$ in the first, second execution, respectively.
    Due to the concentration inequality we have used, we have that
    \[
        \sum_{x \in \calX_1 \cap \calX_2} |\hat{p}^i_x - p_x| \leq \frac{\varepsilon \rho}{4}, i \in \{1,2\} \,,
    \]
    with probability at least $1-\delta.$ Under this event, using the triangle inequality, this means that
    \[
        \sum_{x \in \calX_1 \cap \calX_2} |\hat{p}^1_x - \hat{p}^2_x| \leq \frac{\varepsilon \rho}{2}, i \in \{1,2\} \,,
    \]
    Notice that since pick a number uniformly at
    random from an interval with range $\eps,$ for some given $x \in \calX_1 \cap \calX_2$,
    we have that $\Pr[x \in  L_1\setminus L_2 \cup L_2\setminus L_1] \leq |\hat{p}_x^1 - \hat{p}_x^2|/\eps.$ Thus, taking a union bound over $x \in \calX_1 \cap \calX_2$, we see that
    \[
       \Pr[\calE] \leq \frac{\sum_{x \in \calX_1 \cap \calX_2} |\hat{p}^1_x - \hat{p}^2_x|}{2\eps} \leq \frac{\eps\rho}{2\eps} = \frac{\rho}{2}\,.
    \]
    Putting everything together, we see that the probability that the two outputs of the algorithm
    differ is at most $\delta + \delta/2 + \rho/2 < \rho.$
\end{proof}

\subsection{Replicable Agnostic PAC Learner for Finite $\calH$}
\label{apx:replicable agnostic pac learner}
In this section we present a replicable agnostic PAC learner for finite hypothesis classes,
i.e., a learner whose output is a hypothesis that has error rate close to the best one in the class. Our construction relies on the replicable SQ oracle from \cite{impagliazzo2022reproducibility} (see \Cref{thm:replicable sq learner}).
The idea is simple: since the error rate of every $h \in \calH$ 
can be replicably estimated using \Cref{thm:replicable sq learner}, we do
that for every $h \in \calH$ and then we return the one that has the smallest 
estimated value. 

\begin{algorithm}[ht!]
\caption{Replicable Agnostic Learner for Finite $\calH$}
\label{algo:replicable agnostic learner}
\begin{algorithmic}[1]
\State \texttt{Input: Hypothesis class $\calH$, sample access to a distribution $\calD$ over $\calX \times \{0,1\}$}
\State \texttt{Parameters: accuracy $\eps$, confidence $\delta$, replicability $\rho$}
\State \texttt{Output: Classifier $h$ that is $\eps$-close to the best one in $\calH$
and its estimated error}
 on $\calD$
\State $\hat{\alpha}_h \gets \mathrm{ReprErrorEst}(\eps/2,\delta/|\calH|,\rho/|\calH|), \forall h \in \calH$    \Comment{\Cref{thm:replicable sq learner}.} 
\State $\hat{h}^* \gets \arg\min_{h \in \calH} \hat{a}_h$ \Comment{Break ties arbitrarily in a consistent manner.}
\State Output $(\hat{h}^*, \hat{\alpha}_{\hat{h}^*})$
\end{algorithmic}
\end{algorithm}

It is not hard to see that \Cref{algo:replicable agnostic learner} is $\rho$-replicable and
returns a hypothesis whose error is $\eps$-close to the best one.
\begin{claim}\label{clm:replicable agnostic learner}
    Let $\calH$ be a finite hypothesis class and $\eps, \delta, \rho \in (0,1)^3$. Given $ O\left(\frac{|\calH|^3}{\eps^2\rho^2} \log\left(\frac{|\calH|}{\delta}\right)\right)$ i.i.d. samples from $\calD$, \Cref{algo:replicable agnostic learner} is $\rho$-replicable
    and returns a classifier $\hat{h}^*$
    with $\mathrm{err}(\hat{h}^*) < \min_{h \in \calH} \mathrm{err}(h) + \eps$,
    with probability at least $1-\delta.$
\end{claim}
\begin{proof}
    The replicability of the algorithm follows from the fact that we estimate
    each $\hat{a}_h$ replicably with parameter $\rho/|\calH|$ and we make $|\calH|$ 
    such calls. 

    Notice that for each call to the replicable error estimator
    we need $n_h = O\left(\frac{|\calH|^2}{\eps^2\rho^2} \log\left(\frac{|\calH|}{\delta}\right)\right)$ samples and we 
    make $|\calH|$ such calls.

    Since the accuracy parameter of the statistical 
    query oracle is $\eps/2$, using the triangle inequality, we have that $|\hat{a}_{\hat{h}^*} - \min_{h \in \calH}a_h| \leq \eps.$

    Finally, the correctness of the algorithm follows from a union bound over the 
    correctness of every call to the oracle.
\end{proof}

\section{TV Indistinguishability and Replicability}
In this section, we will study the connection between TV indistinguishability and replicability.


\subsection{The Proof of \Cref{lem:TV stability to replicability}}
\label{sec:proof of tv stability to replicability}

We are now ready to establish the connection between
TV indistinguishability and replicability.
The upcoming result is particularly useful because it provides a \emph{data-independent} way to
couple the random variables. 
\begin{proof}[Proof of \Cref{lem:TV stability to replicability}]
Let $\calR$ be Poisson point process with intensity
$\calP \times \mathrm{Leb} \times \mathrm{Leb}$, 
where $\mathrm{Leb}$ is the Lebesgue measure over $\reals_+$ (cf. \Cref{thm:pairwise opt coupling protocol}, \Cref{figure:poisson}). The learning rule $A'$
is defined in the following way. For every $S \in (\{\calX \times \{0,1\})^n$,
let $r = \{(h_i,y_i,t_i)\}_{i \in \nats}$ be an infinite sequence
of the Poisson point process $\calR$
and let $j = \arg\min_{i \in \nats}\{t_i: f_S(h_i) > y_i\}$. The output of $A'$
is $h_j$ and we denote it by $A'(S,r)$. We will shortly explain why this is well-defined, except for a measure zero event. The fact
that $A'$ is equivalent to $A$ follows from the coupling guarantees of this process
(cf. \Cref{thm:pairwise opt coupling protocol}). In particular,
we can instantiate this result with the single random variable $\{A(S)\}$. 
We can now observe that,
except for a measure zero event, (i) since $A$ is absolutely 
continuous with respect to $\calP$, there exists such a density $f_S$, (ii) the set
over which we are taking the minimum is not empty, (iii) the minimum is 
attained at a unique point. This means that $A'$ is well-defined, except for a measure zero event\footnote{Under the measure zero event that at least one of these three conditions does not hold, we let $A'(S,r)$
 be some arbitrary classifier.}, and, by the correctness of the rejection sampling process \cite{angel2019pairwise}, $A'(S)$ has the desired probability distribution.

    We now prove that $A'$ is replicable. Since $A$ is $\rho$-$\TV$ indistinguishable, it follows that 
    \[\E_{S,S' \sim \calD^n}[\dtv(A(S),A(S'))] \leq \rho.\]
    We have shown that $A'$ is equivalent to $A$, so we can see that $\E_{S,S' \sim \calD^n}[\dtv(A'(S),A'(S'))] \leq \rho$. Thus, using the guarantees of \Cref{thm:pairwise opt coupling protocol},
    we have that for any datasets $S,S'$
    \[
    \Pr_{r \sim \calR}[A'(S,r) \neq A'(S',r)]
    \leq
    \frac{2\dtv(A'(S),A'(S'))}{1 + \dtv(A'(S),A'(S'))}\,.
    \]
    By taking the expectation over $S,S'$, we get that
    \begin{align*}
        \E_{S,S' \sim \calD^n}\left[\Pr_{r \sim \calR}[A'(S,r) \neq A'(S',r)]\right] &\leq
        \E_{S,S' \sim \calD^n}\left[\frac{2\dtv(A'(S),A'(S'))}{1 + \dtv(A'(S),A'(S'))}\right] \\
        &\leq \frac{2 \E_{S,S' \sim \calD^n}\left[\dtv(A'(S),A'(S'))\right]}{1 + \E_{S,S' \sim \calD^n}\left[\dtv(A'(S),A'(S'))\right]} \\
        &\leq \frac{2 \rho}{1 + \rho},
    \end{align*}
    where the first inequality follows from \Cref{thm:pairwise opt coupling protocol} and taking the
    expectation over $S, S'$, the second inequality follows from Jensen's inequality, and
    the third inequality follows from the fact that $f(x) = 2x/(1+x)$ is increasing. Now notice that 
    since the source of randomness $\calR$ is independent 
    of $S,S'$, we have that
    \[
         \E_{S,S' \sim \calD^n}\left[\Pr_{r \sim \calR}[A'(S,r) \neq A'(S',r)]\right] = \Pr_{S,S'\sim \calD^n, r \sim \calR}[A'(S,r) \neq A'(S',r)] \,.
    \]
    Thus, we have shown that
    \[
        \Pr_{S,S'\sim \calD^n, r \sim \calR}[A'(S,r) \neq A'(S',r)]
        \leq \frac{2\rho}{1+\rho} \,,
    \]
    so the algorithm $A'$ is $n$-sample $\frac{2\rho}{1+\rho}$-replicable, which concludes the proof.
\end{proof}

\subsection{A General Equivalence Result}
\label{app:general equivalence}

In this section, we focus on the following two stability/replicability definitions.

\begin{definition}
[Replicability \cite{impagliazzo2022reproducibility}]
Let $\calR$ be a distribution over random strings.
A learning rule $A$ is 
$\rho$-replicable if 
for any distribution $\mu$ over $\calI$ 
and two independent instances 
$I, I' \sim \mu$ it holds that
\[
\Pr_{I,I'\sim \mu, r \sim \calR}[A(I, r) \neq A(I', r) ] \leq \rho\,.
\]
\end{definition}

\begin{definition}
[Total Variation Indistinguishability]
A learning rule $A$ is $\rho$-$\TV$ indistinguishable if for any distribution $\mu$ and two independent instances $I,I' \sim \mu$ it holds that
\[
\E_{I,I' \sim \mu} [ \dtv( A(I), A(I') ) ] \leq \rho\,.
\]
A learning rule $A$ is $\rho$-fixed prior $\TV$ indistinguishable if for any distribution $\mu$, there exists some prior $\calP_\mu$ 
such that for $I \sim \mu$ it holds that
\[
\E_{I \sim \mu} [ \dtv( A(I), \calP_\mu ) ] \leq \rho\,.
\]
\end{definition}

\begin{definition}
[Pseudo-Global Stability]
Let $\calR$ be a distribution over random strings.
A learning rule $A$ is said to be 
$(\eta, \nu)$-pseudo-globally stable if for any distribution $\mu$
there exists an element $o_r \in \calO$ for every $r \in \mathrm{supp}(\calR)$ (depending on $\mu$) such
that
\[
    \Pr_{r \sim \calR} \left[\Pr_{I \sim \mu}[ A(I, r) = o_r ] \geq \eta \right] \geq \nu\,.
\]    
\end{definition}

Our general equivalence result follows.
\begin{proposition}
[TV Indistinguishability $\equiv$ Replicability]
Let $\calI$ be an input space and $\calO$ be an output space.
\begin{itemize}
    \item If a learning rule $A$ is $\rho$-replicable, 
then it is also $\rho$-$\TV$ indistinguishable.
    \item Consider a prior distribution $\calP$ over $\calO$. Consider a learning rule $A$ that is $\rho$-$\TV$ indistinguishable 
and absolutely continuous with respect to $\calP$. Then, there
exists a learning rule $A'$ that is equivalent to $A$ and $A'$ is $2\rho/(1+\rho)$-replicable.
\end{itemize}
\end{proposition}

We remark that one can adapt the proofs of 
\Cref{fact:replicability implies tv stability} and 
\Cref{lem:TV stability to replicability} by setting $\calI = (\calX \times \{0, 1\})^n$, $\mu = \calD^n$ and $\calO = \{0,1\}^\calX$. Moreover, when $\calI$ is countable, the 
design of the reference probability measure works 
in a similar way. Hence, we get the following corollary.

\begin{corollary}\label{cor:TV stability to replicability
for countable X general I}
    Let $\calI$ be a countable domain and let $A$
    be a learning rule that is $\rho$-$\TV$
    indistinguishable. Then, there exists a $\frac{2\rho}{1+\rho}$-replicable
    learning rule $A'$ that is equivalent to $A.$
\end{corollary}

\section{TV Indistinguishability and Differential Privacy}

\subsection{DP Preliminaries}
We introduce
some standard tools from the DP literature. We start with the Stable Histograms
algorithm \cite{korolova2009releasing, bun2016simultaneous}. Let $\calX$ be some domain
and let $S \in \calX^n$ be a (multi)set of its elements. We denote
by $\text{freq}_S(x) = \frac{1}{n}\cdot |\{i\in[n]:x_i = x\}|,$ i.e.,
the fraction of times that $x$ appears in $S.$ The following result holds. It essentially allows us to privately publish a short list of elements that appear with high frequency in a dataset.

\begin{lemma}[Stable Histograms \cite{korolova2009releasing, bun2016simultaneous}]
\label{lem:stable histograms}
    Let $\calX$ be some domain. For 
    \[
        n \geq O\left( \frac{\log(1/(\eta\beta\delta))}{\eta\varepsilon} \right)
    \]
    there exists an $(\varepsilon,\delta)$-differentially private algorithm $\texttt{StableHist}$ which, with probability
    at least $1-\beta,$ on input $S = (x_1,\ldots,x_n) \in \calX^n$, outputs a list 
    $L \subseteq \calX$ and a sequence of estimates $a \in [0,1]^{|L|}$
    such that
    \begin{itemize}
        \item Every $x$ with $\mathrm{freq}_S(x) \geq \eta$ appears in $L$.
        \item For every $x \in L,$ the estimate $a_x$ satisfies $|a_x - \mathrm{freq}_S(x)| \leq \eta.$
    \end{itemize}
\end{lemma}

We also recall the agnostic private learner for finite classes that was
proposed in \cite{kasiviswanathan2011can} and is based on the Exponential Mechanism of \cite{mcsherry2007mechanism}.
\begin{lemma}[Generic Private Learner \cite{kasiviswanathan2011can}]
\label{lem:agnostic learner}
    Let $\calH \subseteq \{0,1\}^\calX$.
    There is an $(\varepsilon,0)$-differentially private algorithm $\texttt{GenPrivLearner}$ which given
    \[
        n = O\left(\log(|\calH|/\beta) \cdot \max\left\{\frac{1}{\varepsilon\alpha},\frac{1}{\alpha^2}\right\}\right)
    \]
    samples from $\calD$, outputs a hypothesis $h$ such that
    \[
        \Pr\left[ \err_\calD(h) \leq \min_{h' \in \calH} \err_\calD(h') + \alpha \right] \geq 1-\beta \,.
    \]
\end{lemma}

Finally, we state a result relating weak learners and privacy.
\begin{theorem}
[Weakly Accurate Private Learning $\implies$ Finite Littlestone Dimension \cite{alon2019private}]
\label{thm:dp learnable implies finite Littlestone}
Let $\calX$ be some domain and $H \subseteq \{0,1\}^\calX$ be a hypothesis class
with Littlestone dimension $d \in \nats \cup \{\infty\}$ and let $A$ be a weakly accurate learning algorithm (i.e., $(\alpha, \beta)$-accurate with $\alpha = 1/2 - \gamma, \beta = 1/2-\gamma$) for $H$
with sample complexity $n$ that satisfies $(\varepsilon,\delta)$-differential privacy with $(\varepsilon, \delta) = (0.1, 1/(n^2 \log(n)))$.
Then,
$n \geq \Omega(\log^\star(d))$.

In particular any class that is privately weakly-learnable has a finite Littlestone dimension.
\end{theorem}

We remark that this theorem appears in \cite{alon2019private} with accuracy
constant $0.1.$ However, it is known from \cite{dwork2010boosting} that a DP algorithm
with error $1/2-\gamma$ can be boosted to one with arbitrarily small error
with negligible loss in the privacy guarantees.

The following result that appears in~\cite{bun2020equivalence}
shows that if $\calH$ has finite Littlestone dimension,
then there exists a $\rho$-globally stable learner for this class.
\begin{theorem}[Finite Littlestone Dimension $\implies$ Global Stability~\cite{bun2020equivalence}]\label{thm:finite littlestone
implies global stability}
    Let $\calX$ be some domain and $\calH \subseteq \left\{ 0,1\right\}^\calX$ be a hypothesis class with Littlestone dimension $d < \infty.$ Let $\alpha > 0$ be the accuracy parameter and define $n = 2^{2^{d+2}+1}4^{d+1}\cdot \left\lceil\frac{2^{d+2}}{\alpha} \right\rceil$. Then, 
    there exists a randomized algorithm $A:(\calX \times \{0,1\})^n \times \calR \rightarrow \{0,1\}^X$ such that for any realizable distribution $\calD$ there exists a hypothesis $f_\calD$ for
    which
    \[
        \Pr_{S \sim \calD^n, r \sim \calR}[A(S,r) = f_\calD] \geq \frac{1}{(d+1)2^{2^d+1}},\quad \Pr_{(x,y)\sim \calD}[f_\calD(x) \neq y] \leq \alpha \,,
    \]
    where $\calR$ is the source of internal randomness of $A.$
\end{theorem}

{We also include a result from \cite{ghazi2021user,bun2023stabilityIsStablest} which states that replicability
implies differential privacy under general input domains\footnote{In fact, this result holds for general statistical tasks. The parameters stated in \cite{bun2023stabilityIsStablest} are slightly looser, but using our boosting results we can generalize them and use the ones that appear in the statement.}.

\begin{theorem}[Replicability $\implies$ Differential Privacy \cite{ghazi2021user, bun2023stabilityIsStablest}]\label{thm:repl implies DP general tasks domains}
    Let $\calH \subseteq \{0,1\}^\calX$, where $\calX$ is some input domain. If
    $\calH$ is learnable by an $n$-sample $(\alpha,\beta)$-accurate $\rho$-replicable learner $A$, for $\alpha \in (0,1/2), \rho \in (0,1), \beta \in \left(0, \frac{2\rho}{\rho+1} - 2\sqrt{\frac{2\rho}{\rho+1}} + 1\right)$, then, for any $(\alpha',\beta',\eps,\delta) \in (0,1)^4$ it is learnable by
    an $(\alpha+\alpha', \beta')$-accurate $(\varepsilon, \delta)$-differentially private learner. Moreover, its sample complexity is 
    \[
        n \cdot \poly(1/\alpha',1/\varepsilon, \log(1/\delta),\log(1/\beta')) \,.
    \]
\end{theorem}
}

{
\subsection{The Proof of \Cref{prop:dp to stab}}
\label{sec:dp to tv}
In this section we show that Global Stability (cf. \Cref{def:global-stability}) implies TV indistinguishability in the context of PAC
learning. In particular, we show that given black-box access
to a $\rho$-globally stable learner $A$ whose stable
output is $\alpha$-accurate, e.g., the one described in \Cref{thm:finite littlestone
implies global stability}, we can transform
it to a $\rho$-TV indistinguishable learner which is $(\alpha+\alpha',\beta)$-accurate, with a multiplicative $\poly(1/\rho,1/\alpha',\log(1/\beta)$ blow-up in its sample 
complexity. We remark that this transformation is not restricted
to countable domains $\calX$. As an intermediate result, we show that global 
stability implies replicability.

\begin{lemma}[Global Stability $\implies$ Replicability]\label{lem:global stability implies replicability}
    Let $A$ be an $n$-sample $\rho$-globally stable learner whose stable hypothesis
    is $\alpha$-accurate. Then, for every $\rho', \alpha', \beta \in (0,1)^3$, there
    exists a learner $A'$ (\Cref{algo:global stability to replicability}) that is $\rho'$-replicable and $(\alpha+\alpha',\beta)$-accurate. Moreover, $A'$ needs
    \[
        \widetilde{O}\left(\frac{\log(1/\beta)}{\rho'^2\rho^3} \right) 
    \] oracle calls to $A$
    and uses 
    \[
        \widetilde{O}\left(\frac{\log(1/\beta)}{\rho^2\rho'^3}\cdot \left(n+\frac{1}{\alpha'^2} \right)  \right) 
    \] samples.
\end{lemma}

\begin{proof}
    We first argue about the accuracy and the confidence of the algorithm. Let $h_A$ be the hypothesis such that $\Pr_{S \sim \calD^n}[A(S) = h]\geq 1-\rho$. The replicable heavy hitters algorithm (\Cref{algo:replicable heavy hitters}) guarantees that, with probability at least $1-\beta/2$, $h_A$ will be contained
    in the output list $L$ (\Cref{lem:replicable heavy hitters}). We call this event $E_0$ and we condition on it.
    In the next step, we call the replicable agnostic learner on $L$ (\Cref{algo:replicable agnostic learner}). Since there is a hypothesis whose error
    rate is at most $\alpha$, we know that the output of the agnostic learner will
    have error rate at most $\alpha+\alpha'$, with probability at least $1-\beta/2$ (\Cref{lem:agnostic learner}). Let us call this event $E_1$. Thus, we see that
    by taking a union bound over the probabilities of these two events,
    the error rate of the output of our algorithm will be at most $\alpha+\alpha'$,
    with probability at least $1-\beta$.

    We now shift our focus to the replicability of our algorithm. First, notice
    that because of the guarantees of the replicable heavy hitters (\Cref{lem:replicable heavy hitters}) the list $L$ will be the same
    across two executions when the randomness is shared, with probability at 
    least $1-\rho'/2$. Let us call this event $E_2$. Similarly, under the event
    $E_2$, the output of the agnostic learner will be the same across two executions
    with probability $1-\rho'/2.$ Let us call this event $E_3.$ By taking a
    union bound over $E_2, E_3$, we see that the algorithm is $\rho'$-replicable.

    The sample complexity of the algorithm follows by the sample complexity of the
    replicable heavy hitters and the replicable agnostic learner (\Cref{lem:replicable heavy hitters}, \Cref{lem:agnostic learner}). In particular, 
    we need
    \[
        \widetilde{O}\left( n\cdot\frac{\log(1/\beta)}{\rho^2\rho'^3}\right) \,,
    \]
    samples for this step and since the list has size $O(1/\rho')$ we need
    \[
        \widetilde{O}\left(\frac{\log(1/\beta)}{\alpha'^2\rho^2\rho'^3}\right) \,,
    \]
    for the replicable agnostic learner.
\end{proof}

\begin{algorithm}[ht]
\caption{From Global Stability to Replicability}
\label{algo:global stability to replicability}
\begin{algorithmic}[1]
\State \texttt{Input:  
Black-box access to a $n$-sample $\rho$-globally stable learner $A$ with $\alpha$-accurate stable hypothesis, sample access to distribution $\calD$}
\State \texttt{Parameters: $\rho', \alpha',\beta \in (0,1)^3$}
\State \texttt{Output: Classifier $h: \calX \rightarrow \{0,1\}$}
\State $\calD' \gets $ distribution induced by drawing $S \sim \calD^n$ and running $A(S)$
 \State $L \gets $ output of \texttt{ReplicableHeavyHitters} (\Cref{algo:replicable heavy hitters}) with threshold $\rho/2$, error $\rho/4$, confidence $\beta/2$, replicability $\rho'/4$
\State Output $\texttt{AgnosticReplicableLearner}$ (\Cref{algo:replicable agnostic learner}) on hypothesis class $L$, with accuracy $\alpha'$, confidence $\beta'/2$ and replicability $\rho'/2$ 
\end{algorithmic}
\end{algorithm}

\begin{corollary}\label{cor:global stability implies tv ind}
    Let $A$ be an $n$-sample $\rho$-globally stable learner whose stable hypothesis
    is $\alpha$-accurate. Then, for every $\rho', \alpha', \beta \in (0,1)^3$, there
    exists a learner $A'$ (\Cref{algo:global stability to replicability}) that is $\rho'$-$\TV$ indistinguishable and $(\alpha+\alpha',\beta)$-accurate. Moreover, $A'$ needs
    \[
        \widetilde{O}\left(\frac{\log(1/\beta)}{\rho'^2\rho^3} \right) 
    \] oracle calls to $A$
    and uses 
    \[
        \widetilde{O}\left( \frac{\log(1/\beta)}{\rho^2\rho'^3}\cdot \left(n+\frac{1}{\alpha^2} \right)\right) 
    \] samples.
\end{corollary}
\begin{proof}
    The proof follows immediately from \Cref{lem:global stability implies replicability} and the fact that a $\rho'$-replicable algorithm is $\rho'$-TV indistinguishable (\Cref{fact:replicability implies tv stability}).
\end{proof}

We now explain how we can use the previous
results in the previous section to design
a replicable algorithm for a class $\calH \subseteq \{0,1\}^\calX$ when we know
that $\calH$ admits a DP learner, for general domains $\calX.$
Formally, we prove the following result.
\begin{lemma}[Differential Privacy $\implies$ Replicability in General Domains]\label{lem:dp implies replicability in general domains}
    Let $\calH \subseteq \{0,1\}^\calX$ be a hypothesis class, where $\calX$ is some input domain. Let $A$ be an $n$-sample $(0.1,1/(n^2\log(n)))$-differentially private $(1/2-\gamma,1/2-\gamma)$-accurate learner for $\calH$, for some $\gamma \in (0,1/2]$. Then, for every $\rho, \alpha, \beta \in (0,1)^3$ there
    exists a learner $A'$ that is $\rho$-replicable and $(\alpha,\beta)$-accurate. Moreover, $A'$ uses 
    \[
        \widetilde{O}\left( \frac{(d+1)^3 2^{3\cdot(2^d+1)}\log(1/\beta)}{\rho^2}\cdot \left(2^{2^{d+2}+1}4^{d+1}\cdot \left\lceil\frac{2^{d+2}}{\alpha} \right\rceil+\frac{1}{\alpha^2} \right)\right) 
    \] samples, where $d$ is the Littlestone dimension of $\calH$.
\end{lemma}
\begin{proof}
    The first step in the proof is to notice that the existence
    of such a DP learner for $\calH$ implies that its Littlestone
    dimension $d$ is finite (\cite{alon2019private}, \Cref{thm:dp learnable implies finite Littlestone}). Then, we instantiate \Cref{algo:global stability to replicability} with the
    globally stable algorithm from \cite{bun2020equivalence} (\Cref{thm:finite littlestone
implies global stability}) with accuracy $\alpha/2$. Notice that since the random bits
for the globally stable need to be different across two executions of the algorithm, we use two different sources of randomness, one that is public, i.e., shared across two executions, and one that is private, i.e., not shared across two executions. Due to \Cref{lem:private replicability}, this
is equivalent to the original definition of replicability (\Cref{def:replicability}). For the remaining two steps, i.e., the replicable heavy-hitters and the replicable agnostic learner, we use
public random bits. The sample complexity of the algorithm
follows from the sample complexity of \Cref{thm:finite littlestone
implies global stability} and \Cref{lem:global stability implies replicability}.
\end{proof}

\begin{corollary}[Differential Privacy $\implies$ TV Indistinguishability in General Domains]\label{cor:dp implies TV ind in general domains}
    Let $\calH \subseteq \{0,1\}^\calX$ be a hypothesis class, where $\calX$ is some input domain. Let $A$ be an $n$-sample $(0.1,1/(n^2\log(n)))$-differentially private $(1/2-\gamma,1/2-\gamma)$-accurate learner for $\calH$, for some $\gamma \in (0,1/2]$. Then, for every $\rho, \alpha, \beta \in (0,1)^3$ there
    exists a learner $A'$ that is $\rho$-$\TV$ indistinguishable and $(\alpha,\beta)$-accurate. Moreover, $A'$ uses 
    \[
        \widetilde{O}\left( \frac{(d+1)^3 2^{3\cdot(2^d+1)}\log(1/\beta)}{\rho^2}\cdot \left(2^{2^{d+2}+1}4^{d+1}\cdot \left\lceil\frac{2^{d+2}}{\alpha} \right\rceil+\frac{1}{\alpha^2} \right)\right) 
    \] samples, where $d$ is the Littlestone dimension of $\calH$.
\end{corollary}

\begin{proof}
    The proof of this result follows immediately by \Cref{lem:dp implies replicability in general domains} and the fact that 
    replicable learners are also TV indistinguishable learners (\Cref{thm:replicability-implies-TV-ind}).
\end{proof}
}

\subsection{List-Global Stability $\implies$ TV Indistinguishability}
\label{sec:list-global stability to tv}
{
In this section we provide a different TV indistinguishable learner
for classes with finite Littlestone dimension that has polynomial
sample complexity dependence on the Littlestone dimension of the class. 
This learner builds upon the results of \cite{ghazi2021sample, ghazi2021user}.
In particular, \cite{ghazi2021sample} show that a class with finite Littlestone
dimension admits a list-globally stable learner (\Cref{def:list global stability}). This learner constructs a sequence of hypothesis classes whose Littlestone dimension is at most that of $\calH$, and part of the proof requires
that uniform convergence (\Cref{def:uc}) holds for all of them. In order to avoid
making measurability assumptions on the domain $\calX$ and the hypothesis class $\calH$ that would imply such a claim, we only
state the results for countable $\calH.$ Nevertheless, we emphasize that they hold for 
more general settings.
}

We underline that the result in~\cite{ghazi2021user}
which designs a pseudo-globally stable learner for classes
with finite Littlestone dimension,
holds in the setting
where $\calX$ is finite because it relies on correlated sampling. 
The reason behind this fact is that they have to convert a DP learner to a pseudo-globally stable one.
In our case, we have to show that if $\calH$ is learnable by a DP algorithm, it also admits a TV indistinguishable one.
The proof of \Cref{prop:dp to stab} follows almost directly from a result appearing in \cite{ghazi2021user}.

\begin{definition}[List-Global Stability \cite{ghazi2021user}]\label{def:list global stability}
    A learning algorithm $A$ is said to be $m$-sample $\alpha$-accurate $(L,\eta)$-list-globally
    stable if $A$ outputs a set of at most $L$ hypotheses and there exists a hypothesis $h$
    (that depends on $\calD$) such that $\Pr_{(x_1,y_1),\ldots,(x_m,y_m) \sim \calD^n}[h \in 
    A((x_1,y_1),\ldots,(x_m,y_m))] \geq \eta$ and $\err_{\calD}(h) \leq \alpha.$
\end{definition}

\cite{ghazi2021user} showed the following result regarding list $m$-list-globally stable learners, which is a modification of a result of \cite{ghazi2021sample}.

\begin{lemma}[Finite Littlestone $\Rightarrow$ List-Global Stability \cite{ghazi2021user, ghazi2021sample}]\label{lem:list-globally-stable result}
    Let $\alpha, \zeta > 0.$ and $\calH \subseteq \{0,1\}^\calX$ be a countable hypothesis class with $\mathrm{Ldim}(\calH) = d < \infty$, where $\calX$ is an arbitrary domain.
    Then,
    there is a $(d\log(1/\zeta)/\alpha)^{O(1)}$-sample $\alpha$-accurate $\left(
    \exp\left((d/\alpha)^{O(1)}\right), \Omega(1/d)\right)$-list-globally stable learner
    for $\calH$ such that, with probability at least $1-\zeta$, every hypothesis
    $h'$ in the output list satisfies $\err_{\calD}(h') \leq 2\alpha.$
\end{lemma}

\begin{algorithm}[ht]
\caption{List-Global Stability $\implies$ TV Indistinguishability (Essentially Algorithm 1 in \cite{ghazi2021user})}
\label{algo:finite littlestone dimension implies pseudo-global}
\begin{algorithmic}[1]
\State \texttt{Input: Black-box access to list-globally stable learner $A$}
\State \texttt{Parameters: $\alpha, \beta, \rho, \eta, L$}
\State \texttt{Output: Classifier $h: \calX \rightarrow \{0,1\}$}
\State $\tau \gets 0.5 \eta$
\State $\gamma \gets \frac{10^6 \log(L/(\rho \tau))}{\tau}$
\State $k_1 \gets \frac{10^6 \log(L/(\rho \tau))}{\tau^2}$
\State $k_2 \gets \left\lceil \frac{10^6 \gamma^2\log(L/(\rho\tau))}{\rho^2} \right\rceil$
\State $m \gets (d\log(k_1/\beta)/\alpha)^{O(1)}$ \Comment{Number of samples to run list-globally stable learner with parameters $(\alpha,\beta/k_1)$ (\Cref{lem:list-globally-stable result})}
\For{$i \gets 1$ to $k_1$}
        \State Draw $S_i \sim \calD^m$, run $A$ on $S_i$ to get
        a set $H_i$
\EndFor
\State Let $H$ be the set of all $h \in \calH$ that appear in at least
$\tau \cdot k_1$ of the sets $H_1,\ldots,H_{k_1}$
\For{$j \gets 1$ to $k_2$}
        \State Draw $T_j \sim \calD^m$, run $A$ on $T_j$ to get
        a set $G_j$
\EndFor
\For{$h \in H$}
        \State Let $\hat{Q}_{H,G_1,\ldots,G_{k_2}}(h) = \frac{|\left\{ j \in [k_2]| h \in G_j\right\}|}{k_2}$
\EndFor
\State Let $\hat{\calP}_{H,G_1,\ldots,G_{k_2}}$ be the probability
distribution on $\calH$ defined by
\[
    \hat{\calP}_{H,G_1,\ldots,G_{k_2}}(h) = 
    \begin{cases}
  \frac{\exp(\gamma \hat{Q}_{H,G_1,\ldots,G_{k_2}}(h))}{\sum_{h' \in H}\exp(\gamma \hat{Q}_{H,G_1,\ldots,G_{k_2}}(h'))},  & h \in H, \\
  0, & \text{ otherwise.}
\end{cases}
\]
\State Output $h \sim \hat{\calP}_{H,G_1,\ldots,G_{k_2}}$
\end{algorithmic}
\end{algorithm}

\begin{proposition}[Adaptation from \cite{ghazi2021user}]
\label{lem:littlestone implies tv stable2}
Let $\calH \subseteq \{0,1\}^\calX$ {be a countable hypothesis class} with $\mathrm{Ldim}(\calH) = d < \infty$ and $\calX$ be an arbitrary domain. {
Then, for all $\alpha,\beta, \rho \in (0,1)^3$, there exists 
an $n$-sample $\rho$-$\TV$ indistinguishable algorithm (\Cref{algo:finite littlestone dimension implies pseudo-global}) that is
$(\alpha, 
\beta)$-accurate with respect to the data-generating distribution $\calD$,
where 
\[
    n = \poly(d,1/\alpha,1/\rho,\log(1/\beta))\,.
\]
}
\end{proposition}

\begin{proof}
First, \Cref{lem:list-globally-stable result} guarantees the existence of a list-globally stable learner $A$ for $\calH$.
We will borrow some notation from \cite{ghazi2021user}.
We remark that the proof is a simple adaptation of the proof of Theorem 20 in \cite{ghazi2021user} but we include it for completeness. We will use \Cref{algo:finite littlestone dimension implies pseudo-global} essentially appearing in \cite{ghazi2021user} (this algorithm is the same as Algorithm 1 in \cite{ghazi2021user}; their algorithm has an additional last step which performs correlated sampling). We will show that \Cref{algo:finite littlestone dimension implies pseudo-global} satisfies the conclusion of \Cref{lem:littlestone implies tv stable2} and is the desired TV indistinguishable learner.

\paragraph{Sample Complexity.} The number of samples used by \Cref{algo:finite littlestone dimension implies pseudo-global} is $m \cdot (k_1 + k_2)$, where $m$ is the number of samples used for the black-box list-globally stable learner $A$. In particular, we have that
\[
n(\alpha, \beta, \rho) = \poly(d, 1/\alpha, 1/\rho, \log(1/\beta))\,.
\]

\paragraph{Accuracy Analysis.} By the guarantees of algorithm $A$, we get that the output of $A$ consists only of hypotheses with distributional error at most $\alpha$ with probability $1-\beta/k_1$, a union bound implies that this holds for all hypotheses in $H$ with probability $1-\beta$. This implies the accuracy guarantee for \Cref{algo:finite littlestone dimension implies pseudo-global}.

\paragraph{TV Indistinguishability Analysis.}
Let us set $Q(h) = \Pr_{S \sim \calD^n}[h \in A(S)]$, let $H_{\geq 0.9 \tau} = \{h \in 2^\calX : Q(h) \geq 0.9 \tau \}$ and
$H_{\geq 1.1\tau} = \{h \in 2^\calX : Q(h) \geq 1.1 \tau \}$.
First, \Cref{algo:finite littlestone dimension implies pseudo-global} creates the set $H$ that contains all $h \in \calH$ that appear in at least
$\tau \cdot k_1$ of the realizations $A(S_1),\ldots,A(S_{k_1})$.

The first lemma controls the probability that $H$
contains hypotheses that are ''heavy hitters'' for $A$ and does not contain hypotheses $h$ whose $Q(h)$ is small. 

\begin{lemma}
[Adaptation of Lemma 22 in \cite{ghazi2021user}]
Let $\calE$ denote the good event that $H_{1.1\tau} \subseteq H \subseteq H_{0.9\tau}$. Then $\Pr[\calE] \geq 1- \rho$, where the randomness is over the datasets $S_1,...,S_{k_1}$ and $A$.
\end{lemma}
\begin{proof}
We will first show that $\Pr[H_{1.1\tau} \subseteq H] \geq 1-\rho/2$. Since $A$ outputs a list of size at most $L$, $H_{1.1\tau} \leq \frac{L}{1.1 \tau} \leq L/\tau$.
For any $f \in H_{1.1\tau}$, we have that $1\{f \in H\}$ is an i.i.d. Bernoulli random variable with success probability $Q(f) \geq 1.1\tau$. Hoeffding's inequality implies that
\[
\Pr[f \notin H] \leq \exp(-0.02 \tau^2 k_1) \leq 0.01\rho\tau/L\,.
\]
A union bound over all hypotheses in $H_{1.1\tau}$, implies the desired inequality. The other direction follows by a similar argument and we refer to \cite{ghazi2021user} for the complete argument.
\end{proof}

The next step is to define the distribution $\calP$ with density $\calP(h) \propto \exp(\gamma Q(h)) 1\{h \in H_{\geq 0.9 \tau}\}$. We can also define $\calP_H(h) \propto \exp(\gamma Q((h)) 1\{h \in H\}$, where $\gamma$ is as in \Cref{algo:finite littlestone dimension implies pseudo-global}. The next lemma relates the two distributions.
\begin{lemma}
[Adaptation of Lemma 23 in \cite{ghazi2021user}]
Under the event $\calE$, it holds that $\dtv(\calP, \calP_H) \leq \rho/2$.
\end{lemma}
\begin{proof}
 The proof is exactly the same as the one of Lemma 23 in \cite{ghazi2021user} with the single modification that we pick $\gamma$ to be of different value, indicated by \Cref{algo:finite littlestone dimension implies pseudo-global}.
\end{proof}

Given a list-globally stable learner $A$ (which exists thanks to \Cref{lem:list-globally-stable result}), we can construct the distribution over hypotheses $\wh{\calP}_{H, G_1,...,G_{k_2}}$ appearing in \Cref{algo:finite littlestone dimension implies pseudo-global}. 
We can then relate the empirical distribution $\wh{\calP}_{H,G_1,...,G_{k2}}$ with its population analogue $\calP_H$. 
\begin{lemma}
[Adaptation of Lemma 24 in \cite{ghazi2021user}]
It holds that
$\E[\dtv(\calP_H, \wh{\calP}_{H,G_1,...,G_{k2}})] \leq \rho/2$, where the expectation is over the sets $T_1,...,T_{k_2}$ and the randomness of $A$.
\end{lemma}
\begin{proof}
 The proof is exactly the same as the one of Lemma 24 in \cite{ghazi2021user} with the single modification that we pick $k_2$ to be of different value, indicated by \Cref{algo:finite littlestone dimension implies pseudo-global}.
\end{proof}

Combining the above lemmas (as in \cite{ghazi2021user}), we immediately get that $\E[\dtv(\calP, \wh{\calP}_{H, G_1,...,G_{k_2}})] \leq \rho$, where the expectation is over all the sets $S_1,...,S_{k_1}$ and $T_1,...,T_{k_2}$ given as input to the learner $A$ and $A$'s internal randomness. Note that $\calP$ is independent of the data and depends only on $A$. Both $\calP$ and $\wh{\calP}_{H, G_1,...,G_{k_2}}$ are supported on a finite domain. 
Note that \Cref{algo:finite littlestone dimension implies pseudo-global} that, given a training set $S$, outputs the distribution over hypotheses $\wh{\calP}_{H, G_1,...,G_{k_2}}$ (obtained by \Cref{algo:finite littlestone dimension implies pseudo-global}) satisfies TV indistinguishiability with parameter $2 \rho$ using triangle inequality. Hence,
two independent runs of \Cref{algo:finite littlestone dimension implies pseudo-global} will be $2\rho$-close in total variation in expectation and the algorithm is TV indistinguishable, as promised.


\end{proof}

\subsection{The Proof of \Cref{prop:stab to dp}}
\label{sec:tv to dp}
We are now ready to show that TV indistinguishability implies approximate DP. We first start by
showing that a non-trivial TV indistinguishable learner for a class $\calH$
gives rise to a non-trivial
DP learner for $\calH.$ The algorithm is described in \Cref{algo:TV stable to dp}.
The result then follows from the fact that classes
which admit non-trivial DP learners have finite Littlestone dimension (\Cref{thm:dp learnable implies finite Littlestone}).

\begin{algorithm}[ht]
\caption{From TV Indistinguishability to Differential Privacy}
\label{algo:TV stable to dp}
\begin{algorithmic}[1]
\State \texttt{Input: Black-box access to $(\alpha,\beta)$-accurate 
$\rho$-TV Indistinguishable Learner $A$, Sample $S$}
\State \texttt{Parameters: $\alpha', \beta', \varepsilon, \delta$}
\State \texttt{Output: Classifier $h:\calX \rightarrow \{0,1\}$}
 \State $k \gets O_{\beta,\rho}\left(\frac{\log(\log(1/\beta')/(\beta'\delta))}{\varepsilon}\right), k' \gets O_{\beta,\rho}(\log(1/\beta'))$
\State Break $S$ into disjoint $\{S^j_i\}_{i\in[k],j\in[k']}$ with $|S_i^j| = n, \forall i \in [k], j \in [k']$
 \State $\calP \gets $ data-independent reference probability measure from \Cref{clm:countable X reference measure} \label{step:choice of refernce measure for DP}
\State $(X_1^j, \ldots, X_k^j) \gets \Pi_{\calR}(A(S_1^j), \ldots, A(S_k^j)), \forall j \in [k']$ using the Poisson point process $\calR$ with intensity $\calP \times \text{Leb} \times \text{Leb}$ 
\Comment{$A(S_i^j)$ is a distribution over classifiers, the coupling $\Pi_\calR$ is described in \Cref{thm:pairwise opt coupling protocol}.}
\State Compute list $L^j \gets \texttt{StableHist}(X_1^j, \ldots, X_k^j), $ with $\eta = O_{\beta,\rho}(1/\log(1/\beta')),$ correctness $\beta'/3,$ privacy $(\varepsilon/2,\delta), \forall j \in [k']$ \Comment{\Cref{lem:stable histograms}}
\State $\widetilde{L}^j \gets$ Remove elements from $L^j$ that appear less than $\eta/2$ times, $\forall j \in [k']$ 
\State Output $\texttt{GenPrivLearner}(\widetilde{L}^1,\ldots,\widetilde{L}^{k'})$ with accuracy $\left(\alpha'/2,\beta'/3\right),$ privacy $(\varepsilon/2,0)$ \Comment{\Cref{lem:agnostic learner}}
\end{algorithmic}
\end{algorithm}

Before we state the result formally, let us first
provide some intuition behind the approach. {On a high level it resembles the approaches
of \cite{bun2020equivalence, ghazi2021user, bun2023stabilityIsStablest} to show that (pseudo-)global stability implies differential privacy.}
We consider $k'$ different batches of $k$ datasets of size $n$.
For each such batch, our goal is to couple the $k$ different executions
of the algorithm on an input of size $n$, so that most of these outputs
are, with high probability, the same.
One first approach would be to use a random variable as a ``pivot'' element in each batch:
we first draw $A(S_1^1)$ according to its distribution and the remaining $\{A(S_i^1)\}_{i \in [k]\setminus\{1\}}$
from their optimal coupling with $A(S_1^1),$ given its realized value. Even though this coupling 
has the property that, in expectation, most of the outputs will be the same, 
it is not robust at all. If the adversary changes a point of $S_1^1,$ then the values of all the outputs 
will change! This is not privacy preserving. For this reason, we use the coupling that is 
described in \Cref{thm:pairwise opt coupling protocol}. We use the fact 
that $\calX$ is countable to design a reference probability measure $\calP$ that
is independent of the data. This is the key step that leads to
privacy-preservation.
Then, we can argue that if we follow this approach
for multiple batches, there will be a classifier whose frequency and performance are non-trivial.
The next step is to feed all these hypotheses into the Stable Histograms algorithm
(cf. \Cref{lem:stable histograms}), which will output a list of frequent hypotheses
that includes the non-trivial one we mentioned above.
Finally, we feed these hypotheses into the Generic Private Learner (cf. \Cref{lem:agnostic learner})
and we get the desired result.

\begin{proposition}
\label{lem:TV stable implies dp-learner}
Let $\calH \subseteq \{0,1\}^\calX$ where $\calX$ is countable.
Assume that $\calH$ is learnable by an $(\alpha, \beta)$-accurate $\rho$-$\TV$ indistinguishable learner $A$ using $n_{\TV}$ samples, where $\rho \in (0,1), \alpha \in (0,1/2), \beta \in (0,(1-\rho)/(1+\rho))$.
Then, for any $(\alpha',\beta',\varepsilon, \delta) \in (0,1)^4,$ it is also learnable by an $\left(\alpha + \alpha', \beta'\right)$-accurate $(\varepsilon, \delta)$-differentially private learner and the sample complexity is
\[
n_{\mathrm{DP}} = O_{\beta,\rho}\left(\frac{\log(1/\beta')\cdot\log(\log(1/\beta')/(\beta'\delta))}{\varepsilon} +  \log(1/\eta\beta') \cdot \max\left\{\frac{1}{\varepsilon\alpha'},\frac{1}{\alpha'^2}\right\}\right) \cdot n_{\TV} \,.
\]
\end{proposition}

\begin{proof}
Let $A$ be the TV indistinguishable algorithm.
We need to argue that the output of \Cref{algo:TV stable to dp} is 
$\left(\alpha + \alpha', \beta'\right)$-accurate and $(\varepsilon,\delta)$-DP. We start with the former property.
\paragraph{Performance Guarantee.}
    Let us consider the following experiment. We 
    draw $k$ samples, each one of size $n = n_{\TV}.$ Let $S^1_1,\ldots,S^1_k$ be these samples and $A(S^1_1),\ldots,A(S^1_k)$ be the distributions of the outputs of the algorithm on these samples. We denote by $X^1_i$ the random variable that follows the distribution $A(S^1_i).$ Let us consider a coupling of this collection of variables.
    Then, we have that
    \begin{align*}
        \E_{\text{coupling}}\left[ \min_{j \in [k]} \sum_{i=1}^k \mathbbm{1}_{X^1_i \neq X^1_j}\right] &\leq \E_{\text{coupling}} \left[\sum_{i=1}^k \mathbbm{1}_{X^1_i \neq X^1_1} \right]\\
        &= \sum_{i=1}^k \E_{\text{coupling}} \left[\mathbbm{1}_{X^1_i \neq X^1_1} \right]\\
        &= \sum_{i=1}^k \Pr_{\text{coupling}}[X^1_i \neq X^1_1 ] \,.    
    \end{align*}

    Note that the above hold for any coupling between the random variables $(X^1_i)_{i \in [n]}.$
    Let us  fix the DP parameters $(\eps, \delta)$.
    We will use the coupling protocol of \Cref{thm:pairwise opt coupling protocol}
    with $\Omega = \{0,1\}^\calX$, $\calP$ the probability measure 
    described in \Cref{clm:countable X reference measure}, and $\calR$ the Poisson point process with 
    intensity $\calP \times \mathrm{Leb} \times \mathrm{Leb}$.
    We remark that this choice of $\calP$ satisfies two properties: the collection $A(S^1_i)$ is absolutely continuous with respect to $\calP$  and $\calP$ is data-independent, so it will help us establish the differential privacy guarantees. 
    The guarantees of the coupling of \Cref{thm:pairwise opt coupling protocol} imply that
    \[
        \Pr_{\calR}[X_i^1 \neq X_1^1] \leq \frac{2\dtv(X^1_i,X^1_1)}{1+\dtv(X^1_i,X^1_1)}\,,
    \]
    for all $i \in [k].$ Thus, we have that
    \begin{align*}
        \E_{\calR}\left[ \min_{j \in [k]} \sum_{i=1}^k \mathbbm{1}_{X^1_i \neq X^1_j}\right] \leq \sum_{i=1}^k \frac{2\dtv(X^1_i,X^1_1)}{1+\dtv(X^1_i,X^1_1)}\,.
    \end{align*}
    By taking the expectation over the random draws of the samples $S_1,\ldots,S_k,$ we see that
    \begin{align*}
          \E_{S^1_1,\ldots,S^1_k,\calR}\left[ \sum_{i=1}^k \mathbbm{1}_{X^1_i \neq X^1_1}\right] &\leq \E_{S^1_1,\ldots,S^1_k} \left[ \sum_{i=1}^k \frac{2\dtv(X^1_i,X^1_1)}{1+\dtv(X^1_i,X^1_1)}\right] \\
          &= \sum_{i=1}^k \E_{S^1_1,\ldots,S^1_k} \left[\frac{2\dtv(X^1_i,X^1_1)}{1+\dtv(X^1_i,X^1_1)}\right] \\
          &\leq \sum_{i=1}^k\frac{2\E_{S^1_1,\ldots,S^1_k}[\dtv(X^1_i,X^1_1)]}{1+\E_{S^1_1,\ldots,S^1_k}[\dtv(X^1_i,X^1_1)]} \\
          &\leq \frac{2\rho}{1+\rho}\cdot k,
    \end{align*}
    where the second to last step follows by Jensen's inequality since the function $f(x) = 2x/(1+x)$
    is concave in $(0,1)$ and the last step because $\E_{S^1_1,\ldots,S^1_k}[\dtv(X^1_i,X^1_1)] \leq
\rho$ and $f$ is increasing in $(0,1)$.
    To make the notation cleaner, we let $\rho' = \frac{2\rho}{1+\rho}.$
    Notice that if $\rho < 1$ then $\rho' < 1.$
    Now using Markov's inequality we get that
    \[
        \Pr\left[ \sum_{i=1}^k \mathbbm{1}_{X^1_i \neq X^1_1} \geq \nu k \rho'\right] \leq \frac{1}{\nu} \implies  \Pr\left[ \sum_{i=1}^k \mathbbm{1}_{X^1_i = X^1_1} \geq (1- \nu \rho')k \right] \geq 1 - \frac{1}{\nu} \,,
    \]
    where the probability is with respect to the randomness of the samples
    and the coupling. 

    We denote by $\calE^1_\nu = \left\{\sum_{i=1}^k \mathbbm{1}_{X^1_i = X^1_1} \geq (1-\nu\rho')k  \right\}$ the event that a $(1-\nu\rho')$-fraction 
    of the outputs has the same value. 
    Let us now focus on the number of classifiers in a single experiment that are correct, i.e., their error rate is at most $\alpha < 1/2.$ Let $Y^1_i = \mathbbm{1}_{\err(X^1_i) \geq 1/2}.$ Notice
    that because of the coupling we have used, $\{Y^1_i\}_{i=1}^k$ are not independent, so we cannot simply apply a Chernoff bound to get concentration. Let $\calE_\beta^1$ be the event that the classifier $X^1_1$ is correct. We know that $\Pr[\calE_\beta^1] \geq 1-\beta,$ where
    the probability is taken with respect to the random draws of the input
    and the randomness of the algorithm. Now notice that under the
    event $\calE^1_\nu \cap \calE^1_\beta$ at least $(1-\nu\rho')k$ classifiers
    are correct and have the same output. By a union bound we see
    that 
    \[
        \Pr[\calE^1_\nu \cap \calE_\beta^1] \geq 1 - \beta - \frac{1}{\nu} \,.
    \]
    We now pick $\nu$ so that 
    \[
         1 - \beta - \frac{1}{\nu} = \frac{1-\beta-\rho'}{2} > 0 \implies \nu = \frac{2}{\rho' -\beta+1}\,.
    \]
    Thus, under $\calE^1_\nu \cap \calE^1_\beta$ there are $\frac{1-\beta-\rho'}{1-\beta + \rho'} k$ classifiers that are equal to one another and 
    are correct. We let $q = \frac{1-\beta-\rho'}{1-\beta + \rho'}$. As we discussed, the probability of this event is at least $\frac{1-\beta-\rho'}{2} = p,$ so if we execute it $k'$ times
    we have that with probability at least $1 - e^{-pk'}$ it will occur at least once, i.e., $\Pr\left[\cup_{j \in [k']}\{\calE^j_\nu \cap \calE_\beta^j\}\right] \geq 1 - e^{-pk'}$.
    We 
    pick $k' = 1/p \cdot \log(3/\beta').$ Thus, with probability
    at least $1-\beta'/3$ there is a correct classifier that appears at
    least $qk$ times.  We condition on this event for the rest of proof and we let $S_i^j, X_i^j \sim A(S_i^j)$ be the $i$-th sample,
    classifier of the $j$-th batch, respectively.

    The next step is to feed these classifiers into the Stable 
    Histograms algorithm (cf. \Cref{lem:stable histograms}). We have shown that there 
    exists a good classifier whose frequency is at least $\eta = \frac{q k}{k\cdot k'} = \frac{q}{k'}.$ Thus, our goal
    is to detect hypotheses with frequency at least $\eta/2$. We 
    pick the correctness parameter of the algorithm to be $\beta'/3$
    and the DP parameters to be $(\varepsilon/2,\delta)$. In total, we need
    \[
         n' = O\left(\frac{\log(1/(\eta\beta'\delta))}{\eta\varepsilon}\right) 
          = O\left(\frac{\log(1/\beta')\cdot\log\left(\log(1/\beta')/(qp\beta'\delta)\right)}{qp\varepsilon}\right) \,,
    \]
    hypotheses in our list. Since $n' = k \cdot k'$ it suffices to 
    pick 
    \[
        k = O\left(\frac{\log\left(\log(1/\beta')/(qp\beta'\delta)\right)}{q\varepsilon}\right) \,.
    \]
    Hence, with probability at least $1-\beta'/3$, the output of the 
    algorithm will be a list $L$ that
    contains all the hypotheses with frequency at least $\eta/2$
    along with estimates $a_x$ such that $|a_x - \text{freq}_S(x)| \leq 
    \eta/2.$ Let $x^*$ be the correct and frequent hypothesis whose
    existence we have established. We know that $a_{x^*} \geq \eta/2.$
    Since this algorithm is DP, we can drop from its output all the 
    elements $x \in L$ for which $a_x < \eta/2$ without
    affecting the privacy guarantees. Thus, we end up 
    with a new list $L'$ whose size is $O(1/\eta).$

    The last step of the algorithm is to feed this list into the 
    Generic Private Learner (cf. \Cref{lem:agnostic learner}) with 
    privacy parameters $(\varepsilon/2,0)$ and accuracy parameters $(\alpha'/2,\beta'/3).$ The total number of samples we need for this step is
    \[
        n'' = O\left(\log(1/\eta\beta') \cdot \max\left\{\frac{1}{\varepsilon\alpha'},\frac{1}{\alpha'^2}\right\}\right) \,.
    \]
    Since there is an element in the list whose error is at most $\alpha,$ 
    the guarantees of the algorithm give us that with probability at least $1-\beta'/3$ the output has error at most $\alpha + \alpha'.$

    Thus, by taking a union bound over the correctness of the three steps
    we described, we see that with probability $1-\beta'$ the algorithm outputs a hypothesis whose error is at most $\alpha+\alpha'.$ We now argue 
    that the algorithm is $(\varepsilon,\delta)-$DP.
    \paragraph{Privacy Guarantee.} 

    First we need to show that the coupling step is differentially private. 
    This is a direct consequence of the coupling protocol that we have provided (cf. \Cref{thm:pairwise opt coupling protocol}) and the fact
    that the reference probability measure is data-independent. If the adversary changes an element in $S_i^j, i \in [k], j \in [k'],$ then
    the coupling is robust, in the sense that if we fix the internal
    randomness, then at most one of the elements that the coupling 
    outputs will change. The result for the privacy preservation of this step follows
    by integrating over the internal randomness. 

    For the remaining two steps, i.e., the Stable Histograms 
    and the Exponential Mechanism the privacy guarantee follows
    from their definition. Using the privacy composition, 
    we get that overall our algorithm is $(\varepsilon/2,\delta)+(\varepsilon/2,0)=
    (\varepsilon,\delta)$-differentially private.

\end{proof}

\begin{corollary}\label{cor:TV stable implies finite Littlestone dimension}
Let $\calH \subseteq \{0,1\}^\calX$, where $\calX$ is a countable domain. If $\calH$ is learnable by a $(\alpha, \beta)$-accurate $\rho$-$\TV$ indistinguishable learner using $n_{\TV}$ samples, where $\rho \in (0,1), \alpha \in (0,1/2), \beta \in (0,(1-\rho)/(1+\rho))$, then $\mathrm{Ldim}(\calH) < \infty.$
\end{corollary}
\begin{proof}
    The proof follows directly by combining \Cref{lem:TV stable implies dp-learner} and \Cref{thm:dp learnable implies finite Littlestone}.
\end{proof}

\subsection{Going Beyond Countable $\calX$}\label{apx:dp beyond countable}
We now propose an approach that we believe can lead to a 
generalization of the algorithm beyond countable domains. 
The only change that we make in the algorithm has to do with \Cref{step:choice of refernce measure for DP}, where for every batch $j$ we pick $\calP_j = \frac{1}{k}\sum_{i=1}^k A(S_i^j).$ Notice that 
for every $j \in [k']$ the $\{A(S_i^j)\}_{i\in[k]}$ are absolutely continuous with respect to 
$\calP_j.$ However, it is not immediate now that the choice of $\{\calP_j\}_{j \in [k']}$
leads to a DP algorithm. We believe that it is indeed the case that the algorithm is approximately differentially private
and we leave it as in interesting open problem.

\section{Amplification and Boosting}\label{apx:amplification and boosting}

\subsection{The Proof of \Cref{lem:stability amplification}}
\label{apx:amplification of stability}
Let us first restate the
theorem along with the sample complexity of the algorithm.

\begin{theorem*}
[Indistinguishability Amplification]
Let $\calP$ be a reference probability measure over $\{0,1\}^\calX$ and $\calD$ be a distribution over inputs.
Consider the source of randomness $\calR$ to be a Poisson point process with intensity $\calP \times \mathrm{Leb} \times \mathrm{Leb},$ where
$\mathrm{Leb}$ is the Lebesgue measure over $\mathbb{R}_+$.
Consider a weak learning rule $A$ that is 
(i) $\rho$-$\TV$ indistinguishable with respect to $\calD$ for some $\rho \in (0,1)$,
(ii) $(\alpha,\beta)$-accurate for $\calD$ for $(\alpha, \beta) \in (0,1)^2, \beta < \frac{2\rho}{\rho+1} - 2\sqrt{\frac{2\rho}{\rho+1}} + 1$,  and, (iii)
 absolutely continuous with respect to $\calP$ on inputs from $\calD$.
Then, for any $\rho', \eps, \beta' \in (0,1)^3$, there exists an algorithm $\texttt{IndistAmpl}(A,\calR, \beta',\eps,\rho')$ (\Cref{algo:stability amplification}) that is $\rho'$-$\TV$ indistinguishable with respect to $\calD$ and $(\alpha+\eps,\beta')$-accurate for $\calD$. 
\smallskip
\\
Let $n_A(\alpha,\beta,\rho)$ denote the sample complexity of the weak learning rule $A$ with input $\beta',\eps,\rho'$. 
Then, the learning rule $\texttt{IndistAmpl}(A,\calR, \beta',\eps,\rho')$
uses
\[
        \wt{O}\left(\frac{\log^3\left(\frac{1}{\beta'}\right)}{\left(\frac{2\rho}{\rho+1} - 2\sqrt{\frac{2\rho}{\rho+1}} + 1 -\beta\right)^2\left(1-\sqrt{\frac{2\rho}{\rho+1}}\right)\eps^2\rho'^2} \cdot n_A(\alpha,\beta,\rho) \right) 
    \]
i.i.d. samples from $\calD$.
\end{theorem*}

\begin{algorithm}[ht!]
\caption{Amplification of Indistinguishability Guarantees}
\label{algo:stability amplification}
\begin{algorithmic}[1]
\State \texttt{Input: Black-box access to $(\alpha,\beta)$-accurate 
$\rho$-TV Indistinguishable Learner $A$, Sample access to $\calD$, Access to Poisson point process $\calR$ with
    intensity $\calP \times \mathrm{Leb} \times \mathrm{Leb}$} \Comment{$\calP$ is the reference probability measure from \Cref{clm:countable X reference measure}.}
\State \texttt{Parameters: $\beta', \eps, \rho'$}
\State \texttt{Output: Classifier $h: \calX \rightarrow \{0,1\}$ }
 \State $\eta, \nu \gets \sqrt{\frac{2\rho}{1+\rho}}, \sqrt{\frac{2\rho}{1+\rho}}$
 \State $\calP \gets $ data-independent reference probability measure from \Cref{clm:countable X reference measure}
 \State $k \gets \frac{\log(3/\beta')}{1-\nu-\beta/(1-\eta)}$
\State $r_i \gets $ an infinite sequence of the Poisson Point Process $\calR$, $\forall i \in [k]$ \Comment{cf. \Cref{thm:pairwise opt coupling protocol}.} 
\State $\calD_{r_i} \gets $ the distribution of hypotheses that is induced by $A(S,r_i)$ when $S \sim \calD^n, \forall i \in [k]$
\State $L_i \gets \mathrm{HeavyHitters}\left(\calD_{r_i}, \frac{3}{4}(1-\eta), \frac{1}{4}(1-\eta), \rho'/(2k),\beta'/(3k)\right), \forall i \in [k]$ \Comment{\Cref{algo:replicable heavy hitters}.}
\State $\left(\hat{h}_i, \widehat{\mathrm{err}}(\hat{h}_i)\right) \gets \mathrm{AgnosticLearner}(L_i, \eps/2, \rho'/(2k), \beta'/(3k)), \forall i \in [k]$ \Comment{\Cref{algo:replicable agnostic learner}.}
\For{$i \gets 1$ to $k$}
        \If{$\widehat{\mathrm{err}}(\hat{h}_i) \leq \alpha + \eps/2$}
            \State Output $\hat{h}_i$
        \EndIf
      \EndFor
\State Output the all $1$ classifier
\end{algorithmic}
\end{algorithm}

\begin{proof}
    Since $A$ is $\rho$-TV indistinguishable there is an equivalent learning rule
    $A'$ that is
    $\frac{2\rho}{1+\rho}$-replicable (cf. \Cref{lem:TV stability to replicability}) 
    and uses randomness $\calR$, where $\calR$ is a Poisson point process with
    intensity $\calP \times \mathrm{Leb} \times \mathrm{Leb}$, with $\mathrm{Leb}$
    being the Lebesgue measure over $\reals_+.$
    Let 
    \[
          \calR_\eta = \left\{r \in \calR: \exists h \in \calH \text{ s.t. } \Pr_{S \sim \calD^n}[A'(S,r) = h] \geq 1-\eta \right\} \,,
    \]
  We have that $\Pr_{r \sim \calR}[r \in \calR_\eta] \geq 1 - \nu,$ for
 $\eta = \frac{\frac{2\rho}{1+\rho}}{\nu}, \nu \in \left[\frac{2\rho}{1+\rho},1\right)$ (cf. \Cref{clm:repl implies pseudo-global stability}). For each $r \in \calR_\eta$ let $h_r \in \calH$
    be an element that witnesses its inclusion in $\calR_\eta$\footnote{If there are multiple such elements then we pick an arbitrary one using a consistent rule.}.
    Notice that since $A'$ is $(\alpha,\beta)$-accurate
    there is at most a $\frac{\beta}{1-\eta}$-fraction of $r \in \calR$
    such that $r \in \calR_\eta, \mathrm{err}(h_r) > \alpha.$ 
    Let $\calR_\eta^* = \left\{r \in \calR_\eta: \mathrm{err}(h_r) \leq \alpha \right\}.$ Now notice that $\Pr_{r \sim \calR}[r \in \calR_\eta^*] \geq 1-\nu - \frac{\beta}{1-\eta}.$ Thus, by picking $k = \frac{\log(3/\beta')}{1-\nu-\beta/(1-\eta)}$ i.i.d. samples from $\calR$
    we have that with probability at least $1 - \beta'/3$ there will be some $r_{i^*} \in \calR_\eta^*.$
    We denote this event by $\calE_1$ and we condition on it for the rest of the proof.

    Let us now focus on the call to the replicable heavy hitters subroutine. We have that,
    with probability at least $1-\beta'/(3k)$, every call will return a list that contains
    all the $(1-\eta)$-heavy-hitters and no elements whose mass is less than $(1-\eta)/2$.
    By a union bound, this happens with probability at least $1-\beta'/3$ for all the calls.
    Let us call this event $\calE_2$ and condition on it for the rest of the proof.
    Notice that under these two events, the list $L_{i^*}$ that corresponds to $r_{i^*}$ will
    be non-empty and will contain a classifier whose error is at most $\alpha.$

    We now consider the calls to the replicable agnostic learner. Notice that every list
    that this algorithm takes as input has size at most $\frac{2}{1-\eta}.$ Moreover, 
    with probability at least $1-\beta/3'$, the estimated 
    error of every classifier will be at most $\eps/2$ 
    away from its true error. We call this event $\calE_3$ and condition on it. Hence, for any $\hat{h}_j, j\in [k],$
    that passes the test in the ``if'' statement, 
    we have that $\mathrm{err}(\hat{h}_j) \leq \alpha + \eps.$
    In particular, the call to $L_{i^*}$ will return $\hat{h}_{i^*}$, with estimated error $\widehat{\mathrm{err}}(\hat{h}_i^*) \leq \alpha + \eps/2$, which means that $\mathrm{err}(\hat{h}_i^*) \leq \alpha + \eps.$ Hence, 
    the algorithm will such a classifier and, by a union bound, the total probability that this 
    event happens is at least $1-\beta'.$

    The replicability of the algorithm follows from a union bound over the replicability of
    the calls to the heavy hitters and the agnostic learner (cf. \Cref{lem:replicable heavy hitters}, \Cref{clm:replicable agnostic learner}). {In particular, since we call the replicable heavy hitters algorithm $k$ times with replicability parameter $\rho'/(2k)$ and the replicable agnostic leaner $k$ times with replicability parameter $\rho'/(2k)$, we know that with probability at least $1-\rho'$ all these calls will return the same output across two executions of the algorithm.}

    For the sample complexity notice that each call 
    to the replicable heavy hitters algorithm requires $O\left( \frac{k^2\log(k/\beta'(1-\eta))}{(1-\eta)^3\rho'^2}\right)$ (cf. \Cref{lem:replicable heavy hitters}.)
    Under the events we have conditioned on, we see that $|L_i| = O(1/(1-\eta)), \forall i \in [k]$, hence each call to the agnostic learner requires
    $O\left(\frac{k^2}{(1-\eta)^3\eps^2\rho'^2} \log\left(\frac{k(1-\eta)}{\beta'}\right)\right)$ (cf. \Cref{clm:replicable agnostic learner}).
    Substituting the value of $k$ gives us that
    the sample complexity is at most
    \[
        O\left(\frac{\log^3\left(\frac{\log(1/\beta')}{\beta'\left((1-\eta)(1-\nu)-\beta\right)}\right)}{\left((1-\eta)(1-\nu)-\beta\right)^2(1-\eta)\eps^2\rho'^2} \right) \,.
    \]
    Plugging in the values of $\eta, \nu$ we get the stated bound.
\end{proof}

\subsection{The Proof of \Cref{thm:boosting algorithm}}
Let us first recall the result we need to prove
along with its sample complexity.

\begin{theorem*}
[Accuracy Boosting]
Let $\calP$ be a reference probability measure over $\{0,1\}^\calX$ and $\calD$ be a distribution over inputs.
Consider the source of randomness $\calR$ to be a Poisson point process with intensity $\calP \times \mathrm{Leb} \times \mathrm{Leb},$ where
$\mathrm{Leb}$ is the Lebesgue measure over $\reals_+$.
Consider a weak learning rule $A$ that is 
(i) $\rho$-$\TV$ indistinguishable with respect to $\calD$ for some $\rho \in (0,1)$,
(ii) $(1/2-\gamma,\beta)$-accurate for $\calD$ for some $\gamma \in (0,1/2), \beta \in \left(0,\frac{2\rho}{\rho+1} - 2\sqrt{\frac{2\rho}{\rho+1}} + 1\right)$, and,
(iii) absolutely continuous with respect to $\calP$ on inputs from $\calD$.
Then, for any $\rho', \eps, \beta' \in (0,1)^3$, there exists an algorithm $\texttt{IndistBoost}(A,\calR, \eps)$ (\Cref{algo:boosting}) that is $\rho'$-$\TV$ indistinguishable with respect to $\calD$ 
and $(\eps,\beta')$-accurate for $\calD$.
\smallskip
\\
If $n_A(\gamma,\beta,\rho)$ is the sample complexity of the weak learning rule $A$ with input $\gamma,\beta,\rho$, then $\texttt{IndistBoost}(A,\calR, \eps)$
uses
\[
\wt{O}\left( \frac{n_A(\gamma, \beta' \eps \gamma^2/6,  \rho \eps \gamma^2/(3(1+\rho))) \log(1/\beta')}{\eps^2 \gamma^2} + \frac{\log(1/\beta')}{(2\rho/(1+\rho))^2 \eps^3 \gamma^2} \right)
\]
i.i.d. samples from $\calD$.
\end{theorem*}

\begin{algorithm}[ht!]
\caption{Boosting of Accuracy Guarantee}
\label{algo:boosting}
\begin{algorithmic}[1]
\State \texttt{Input: Black-box access to weak $(\frac{1}{2}-\gamma,\beta)$-accurate 
$\rho$-TV Indistinguishable Learner $A$, Sample $S \sim \calD^n$, Access to Poisson point process $\calR$ with intensity $\calP \times \text{Leb} \times \text{Leb}$} \Comment{$\calP$ is the reference probability measure from \Cref{clm:countable X reference measure}.}
\State \texttt{Target : $\eps, \beta'$}
\State \texttt{Output: Classifier $h: \calX \rightarrow \{0,1\}$
 }
 
\State \texttt{IndistBoost()} \Comment{This algorithm appears in \cite{impagliazzo2022reproducibility}}
\State $\rho' = 2\rho/(1+\rho)$
\State $T = 100/(\eps \gamma^2)$
\State $\mu_1(x) = 1$
\State $n_{w} = n_A\left(\gamma,\frac{\beta'}{3T},\frac{\rho'}{6T}\right)$
\For{$t = 1..T$}
    \State~~~~~$\calD_{\mu_t}(x) = \frac{\mu_t(x) \calD_X(x)}{d(\mu_t)}$
    \State~~~~~$S_t \gets n_w/\eps \cdot \log(T/\beta')$
    \State~~~~~$S_t' \gets \texttt{RejectionSampling}\left(S_t, n_w, \mu_t, \calR_t^{(1)}\right)$
    \State~~~~~$h_t \sim A\left(S_t', \calR_t^{(2)}\right)$
    \State~~~~~Update $\mu_{t+1}(x)$ using smooth boosting trick of \cite{servedio2003smooth}.
    \State~~~~~Draw $S_t'' = O(1/(\rho'^2 \eps^3 \gamma^2))$ i.i.d. samples from $\calD$
    \State~~~~~If $\texttt{IndistingTestMeasure}\left(\mu_{t+1}, S_t'', \calR_t^{(3)}, \rho'/(3T), \beta'/(3T)\right) \leq 2\eps/3$ then \textbf{output} $\sgn\left(\sum_i h_i\right)$
\EndFor

\State \texttt{RejectionSampling}$\left(S_{\mathrm{in}}, \mathrm{size\_out}, \mu, \calR\right)$ 
\State $S_{\mathrm{out}} = \emptyset$
\For{$(x,y) \in S_{\mathrm{in}}$}
\State~~~~~Pick $b \in [0,1]$ using $\calR$
\State~~~~~If $\mu(x) \geq b$ then $S_{\mathrm{out}} \gets \mathrm{append}(S_{\mathrm{out}},(x,y))$
\State~~~~~If $|S_{\mathrm{out}}| > \mathrm{size\_out}$ then \textbf{output} $S_{\mathrm{out}}$
\EndFor
\State \texttt{IndistingTestMeasure($\mu, S, \calR, \rho', \beta$)}
\State Call Algorithm 1 in \cite{impagliazzo2022reproducibility} (see \Cref{thm:replicable sq learner}) with source of randomness $\calR$ and dataset $S$, error $\eps/3$, confidence $\beta$, replicability $\rho$ and query function $\mu$
\end{algorithmic}
\end{algorithm}

\begin{proof}[Proof of \Cref{thm:boosting algorithm}]
In the sample complexity bound of \Cref{thm:boosting algorithm}, we remark that the first term is the number of samples used by the $\texttt{RejectionSampling}$ mechanism (appearing in \cite{impagliazzo2022reproducibility}) in the $T$ rounds and the second term controls the number of samples used for the $\texttt{IndistingTestMeasure}$ procedure (appearing in \cite{impagliazzo2022reproducibility}) for the $T$ rounds (see \Cref{algo:boosting}).
Let $[T] = \{1,...,T\}$.
As in \cite{impagliazzo2022reproducibility}[Theorem 6.1], we consider that the shared randomness between the two executions
consists of a collection of $3T$ tapes with uniformly random bits. We denote the $j$-th tape in round $t$ by $\calR_t^{(j)}$ for $j \in [3]$ and $t \in [T]$. Since $A$ is $n$-sample $\rho$-TV indistinguishable there is an equivalent learning rule $A'$ that is $n$-sample $\frac{2\rho}{1+\rho}$-replicable (cf. \Cref{lem:TV stability to replicability}) and uses randomness $\calR$, where $\calR$ is a Poisson point process with intensity $\calP \times \mathrm{Leb} \times \mathrm{Leb}$, with $\mathrm{Leb}$     being the Lebesgue measure over $\reals_+$. Let us set $\rho' = 2\rho/(1+\rho)$. The boosting algorithm that we provide below interprets the random strings as follows: for any $t \in [T]$, we set $\calR_t^{(2)} = \calR$ (these will be the tapes used by the equivalent learning algorithm $A'$) and the remaining tapes $\calR_t^{(j)}$ corresponds to random samples from the uniform distribution in $[0,1]$ for $j \in \{1,3\}$ (these will be the tapes used by our sub-routines $\texttt{RejectionSampling}$ and $\texttt{IndistingTestMeasure}$.


The boosting algorithm works as follows:
\begin{enumerate}
    \item As in \cite{servedio2003smooth}, it uses a measure $\mu_t$ to assign different scores to points of $\calX$. First, $\mu_1(x) = 1$ for any point.
    We will not delve into the details on how this step works. For details we refer to \cite{servedio2003smooth} (as in \cite{impagliazzo2022reproducibility} since this step is not crucial for the proof).
    \item At every round $t$, the algorithm performs rejection sampling on a fresh dataset $S_t$ using the routine $\texttt{RejectionSampling}$. This algorithm is TV indistinguishable since it uses the source of randomness $\calR_t^{(1)}$ that provides uniform samples in $[0,1]$ (it is actually replicable). 
    \item The part of the dataset that was accepted from this rejection sampling process is given to replicable learner $A'$, which is equivalent to the TV indistinguishable weak learner $A$. This algorithm uses the shared Poisson point process $\calR_t^{(2)}$ with intensity $\calP \times \text{Leb} \times \text{Leb},$ where $\calP$ is the reference probability measure
    from \Cref{clm:countable X reference measure}, and outputs the same hypothesis with probability $1-\rho'/(6T)$.
    \item Then we use the smooth update rule of \cite{servedio2003smooth} to design the new measure $\mu_{t+1}$ for the upcoming iteration. This step is deterministic. 
    \item Last we check whether the boosting procedure is completed. To this end, we check whether $\mu_t$ is in expectation small. This step again uses a uniformly random threshold in $[0,1]$ and so makes use of the source $\calR_t^{(3)}$.
\end{enumerate}

The algorithm runs for $T = \frac{C}{\eps \gamma^2}$ rounds for some numerical constant $C > 0$. Hence, we will assume access to $3T$ tapes of randomness, $T$ with points from the Poisson point process and $2T$ with uniform draws from $[0,1]$. The correctness of the algorithm follows from \cite{servedio2003smooth} and \cite{impagliazzo2022reproducibility}[Theorem 6.1]. As for the TV indistinguishability,
this is implied by the replicability of the whole procedure.
We have that the weak learner $A'$ is called $T$ times with TV indistinguishability parameter $\rho'/(6T)$, the rejection sampler is called $T$ times so that it outputs $\perp$ with probability $\rho'/(6T)$ and the indistinguishable measure tester is $\rho'/(3T)$-$\TV$ indistinguishable and called $T$ times. A union bound gives the desired result. For further details, we refer to \cite{impagliazzo2022reproducibility} since the analysis is essentially the same.

For the failure probability $\beta'$, the algorithm can fail if the rejection sampling algorithm outputs $\perp$, if the weak learner fails, and if the replicable SQ oracle (\Cref{thm:replicable sq learner}) fails. 
We have that the probability that the rejection sampling gives $\perp$ using $n_w/\eps \cdot \log(T/\beta')$ is at most $\beta'/T$ (which can be considered much smaller than $\rho'
/(6T)$).
Since each one of the three probabilities are upper bounded by $\beta'/(3T)$, the indistinguishable boosting algorithm succeeds with probability $1-\beta'$.
\end{proof}

\subsection{Tight Bound Between $\beta, \rho$}\label{apx:boosting tight bound}
As we alluded before, \Cref{lem:TV stable implies dp-learner}
shows that if we have a $\rho$-TV indistinguishable $(\alpha,\beta)$-accurate learner with $\rho \in (0,1), \alpha \in (0,1/2), \beta \in \left(0,\frac{1-\rho}{1+\rho}\right),$ then the class $\calH$ has finite Littlestone dimension. The
reason we need $\beta \in \left(0,\frac{1-\rho}{1+\rho}\right)$ is because, in expectation over the random draws
of the samples and the randomness of the coupling, this is the fraction of the executions of the algorithm that
will give the same output. The results of \cite{angel2019pairwise} show that under certain conditions, if we want
to couple $k$ random variables whose pairwise TV distance is at most $\rho$, then under the pairwise
optimal coupling the probability that the realization of a pair of them differs is $\frac{2\rho}{1+\rho}.$
However, it is unclear what the implication of this result is in the setting we are interested in.

\subsection{Beyond Countable $\calX$}\label{apx:boosting beyond countable domains}
The barrier to push our approach beyond countable $\calX$ is very closely related
to the one we explained in the DP section. To be more precise, it is not clear how one can design
a data-independent reference probability measure $\calP$ when $\calX$ is uncountable. Hence,
one idea would be to use some \emph{data-dependent} probability measure $\calP$. This would affect our algorithm
in the following way: instead of first sampling the random Poisson point process sequence independently of the data,
we first sample $S_1,\ldots,S_k$ and let the reference probability measure be $\calP = \frac{1}{k}\sum_{i=1}^k A(S_i).$
The difficult step is to show that this algorithm is TV indistinguishable. When we consider a 
different execution of the algorithm we let $S_1',...,S_k'$ be the new samples and
$\calP' = \frac{1}{k}\sum_{i=1}^k A(S'_i)$ be the new reference probability measure. A natural approach
to establish the TV indistinguishability property of the algorithm is to try to couple $\calP, \calP'$
and show that under this coupling, the expected TV distance of two executions of the new algorithm is small.
We leave this question open for future work.

\section{TV Indistinguishability and Generalization}
\label{app:gen}
 Recall that in \Cref{prop:ind implies generalization} we claimed that the generalization bound can shave the dependence on the VC dimension by paying an overhead in the confidence parameter. A similar result appears in \cite{impagliazzo2022reproducibility} relating replicability to generalization.  
 We now present its proof.
 
\begin{proof}
[Proof of \Cref{prop:ind implies generalization}]
Let $S$ be a sample from $\calD^n$. 
Since $A$ is $\rho$-TV indistinguishable, it is also 
$\rho$-fixed prior TV indistinguishable and let $\calP_\calD$ be the sample-independent prior. Consider two samples $h_1 \sim A(S)$ and $h_2 \sim \calP_\calD$. We consider the following quantities:
\begin{itemize}
    \item $\wh{L}(h_1) = \frac{1}{n} \sum_{(x,y) \in S}\mathbbm{1}\{h_1(x) \neq y\}$ is the empirical loss of $h_1$ in $S$.
    \item $\wh{L}(h_2) = \frac{1}{n} \sum_{(x,y) \in S}\mathbbm{1}  \{h_2(x) \neq y\}$ is the empirical loss of $h_2$ in $S$.
    \item $L(h_1) = \Pr_{(x,y) \sim \calD}[h_1(x) \neq y]$ is the population loss of $h_1$ with respect to $\calD$.
\end{itemize}
We will show that all these three quantities are close to each other.
First, 
let us consider the space of measurable functions $\calF = \{f : \|f\|_\infty \leq 1\}$.
We have that
\[
\dtv(P,Q) = \sup_{f \in \calF} \left| \E_{x \sim P}[f(x)] - \E_{x \sim Q}[f(x)] \right|\,.
\]
This means that the total variation distance between two distributions is essentially the worst case bounded distinguisher $f$. Since $\wh{L} :\{0,1\}^\calX \to [0,1]$, we have that
\[
\left |\E_{h_1 \sim A(S)} \left[\wh{L}(h_1)\right] - \E_{h_2 \sim \calP_\calD} \left[\wh{L}(h_2)\right] \right| \leq \dtv(A(S), \calP_\calD)\,.
\]
Similarly, we get that
\[
\left |\E_{h_1 \sim A(S)} \left[L(h_1)\right] - \E_{h_2 \sim \calP_\calD} \left[L(h_2)\right] \right| \leq \dtv(A(S), \calP_\calD)\,.
\]
Now, since $A$ is $\rho$-fixed prior TV indistinguishable, using Markov's inequality, 
we have that $\forall \eps_1 > 0,$
\[
\E_{S \sim \calD^n} \left[\left |\E_{h_1 \sim A(S)} \left[\wh{L}(h_1)\right] - \E_{h_2 \sim \calP_\calD} \left[\wh{L}(h_2)\right] \right|\right] \leq \rho \Rightarrow 
\Pr_{S \sim \calD^n}
\left[
\left|\E_{h_1 \sim A(S)} \left[\wh{L}(h_1)\right] - \E_{h_2 \sim \calP_\calD} \left[\wh{L}(h_2)\right] \right| > \eps_1
\right] 
\leq \frac{\rho}{\eps_1}\,.
\]
In a similar manner, we get
\[
\Pr_{S \sim \calD^n}
\left[
\left|\E_{h_1 \sim A(S)} \left[L(h_1)\right] - \E_{h_2 \sim \calP_\calD} \left[L(h_2)\right] \right| > \eps_1 
\right] 
\leq \frac{\rho}{\eps_1}
\]
We note that, since $\calP_\calD$ is sample-independent, we have that the statistic 
\[\E_{h_2 \sim \calP_\calD}[\wh{L}(h_2)] = \frac{1}{n}\sum_{(x,y) \in S}\Pr_{h_2 \sim \calP_\calD}[h_2(x) \neq y]\]
is a sum of independent random variables with expectation
$\E_{h_2 \sim \calP_\calD}[L(h_2)]$.
We can use standard concentration of independent random variables and get
\[
\Pr_{S \sim \calD^n}
\left[\left|
\E_{h_2 \sim \calP_\calD}\left[\wh{L}_S(h_2)\right]
-
\E_{h_2 \sim \calP_\calD}\left[L_\calD(h_2)\right] \right|
\geq \eps_2
\right] 
\leq 2e^{-2n \eps_2^2}\,,
\]
for any $\eps_2 > 0.$ This means that
\[
\Pr_{S \sim \calD^n}
\left[
\left|\E_{h_1 \sim A(S)}\left[\wh{L}_S(h_1)\right]
-
\E_{h_1 \sim A(S)}\left[L_\calD(h_1)\right] \right|
\geq 2 \varepsilon_1 + \varepsilon_2
\right] 
\leq 2\rho/\eps_1 + 2e^{-2n \eps_2^2} \,,
\]
so we have that, with probability at least $1 - 4\rho/\eps - \delta$,
\[
    \left|\E_{h_1 \sim A(S)}\left[\wh{L}_S(h_1)\right]
-
\E_{h_1 \sim A(S)}\left[L_\calD(h_1)\right] \right| \leq \eps + \sqrt{\frac{\ln(2/\delta)}{2n}} \,.
\]
We note that we obtain the result of \Cref{prop:ind implies generalization} by taking $\eps = \sqrt{\rho}$.
\end{proof}

\end{document}